\documentclass[bj]{imsartLai}

\RequirePackage[OT1]{fontenc}
\RequirePackage{amsthm,amsmath,mathrsfs,natbib,amsfonts,graphicx,epsfig}
\RequirePackage[colorlinks,citecolor=blue,urlcolor=blue]{hyperref}

\startlocaldefs
\numberwithin{equation}{section}
\theoremstyle{plain}
\newtheorem{thm}{Theorem}[section]
\theoremstyle{remark}

\newtheorem{prop}{Proposition}[section]
\newtheorem{exmpl}{Example}[section]

\newcommand{\thetav}{\boldsymbol{\theta}}
\newcommand{\Thetav}{\boldsymbol{\Theta}}
\newcommand{\sigmav}{\boldsymbol{\sigma}}
\newcommand{\Sigmav}{\boldsymbol{\Sigma}}

\newcommand{\Av}{\mathbf{A}}
\newcommand{\Xv}{\mathbf{X}}
\newcommand{\Mv}{\mathbf{M}}

\newcommand{\Hv}{\mathbf{H}}

\newcommand{\bv}{\mathbf{c}}
\newcommand{\xv}{\mathbf{x}}
\newcommand{\tv}{\mathbf{t}}
\newcommand{\tXv}{\tilde{\Xv}}
\newcommand{\AIF}{\text{AIF}}
\newcommand{\di}{\text{d}}
\newcommand{\psiv}{\boldsymbol{\psi}}
\newcommand{\no}{\nonumber}
\endlocaldefs

\begin{document}

\begin{frontmatter}
	\title{On the Adversarial Robustness of Multivariate Robust Estimation}
	\runtitle{On the Adversarial Robustness of Multivariate Robust Estimation}
	
	\begin{aug}
		\author{\fnms{Erhan} \snm{Bayraktar}\thanksref{a,e1}\ead[label=e1,mark]{erhan@umich.edu}}
		\and
		\author{\fnms{Lifeng} \snm{Lai}\thanksref{b,e2}\ead[label=e2,mark]{lflai@ucdavis.edu}}
		
		\address[a]{Department of Mathematics, University of Michigan, Ann Arbor, MI 48104.
			\printead{e1}}
		
		\address[b]{Department of Electrical and Computer Engineering, University of California, Davis, CA 95616.
			\printead{e2},
		}
		
		\runauthor{E. Bayraktar and L. Lai}
		
		\affiliation{University of Michigan and University of California, Davis}
		
	\end{aug}
	
	\begin{abstract}
		In this paper, we investigate the adversarial robustness of multivariate $M$-Estimators. In the considered model, after observing the whole dataset, an adversary can modify all data points with the goal of maximizing inference errors. We use adversarial influence function (AIF) to measure the asymptotic rate at which the adversary can change the inference result. We first characterize the adversary's optimal modification strategy and its corresponding AIF. From the defender's perspective, we would like to design an estimator that has a small AIF. For the case of joint location and scale estimation problem, we characterize the optimal $M$-estimator that has the smallest AIF. We further identify a tradeoff between robustness against adversarial modifications and robustness against outliers, and derive the optimal $M$-estimator that achieves the best tradeoff.
	\end{abstract}
	
	
\end{frontmatter}

\section{Ordinary text}

Most of the existing work on robust statistical inference mainly address distributional robustness issues such as outliers or model uncertainties~\citep{Huber:AMS:64,Huber:Book:09,Hampel:Book:86}. As machine learning and statistical inference algorithms are being increasingly used in safety critical applications and security related applications~\citep{huval2015empirical,mlcyber,litjens2016deep,Nelson:2008:EML,Soule:2005:CFS,Stamp:Book:18,Suthaharan:2014:BDC:2627534.2627557,7995936,8460874}, there is a growing interest in investigating the robustness of statistical inference algorithms in adversarial environments. In these adversarial environments, we are facing more severe situations than those addressed in the classic robust statistical inference problems. One such scenario is where an adversary can observe the whole dataset and then devise its attack vector to modify all entries in the data point hoping to cause the maximum inference error or to control the inference results. For example, the adversarial example phenomenon in the deep neural network~\citep{Pimentel:ISIT:17,szegedy2013intriguing,easyfooldnn,Goodfellow:ICLR:15,Carlini:ISSP:17} where an adversary can observe the whole picture and then carefully modify the pixels in the picture with the goal of fooling the classifier. As another example, it was shown in~\citep{Jagielski:ISSP:18} that an adversary can modify the training data so that the model produced by the linear regression is controlled by the adversary. The existence of such powerful adversaries calls for new models and methodologies for adversarially robust inference.

In a canonical statistical inference problem, one infers parameters of interest from given data points $\{\mathbf{x}_1,\cdots,\mathbf{x}_N\}$. The classic robust inference mainly deals with \emph{distributional robustness}, i.e., the shape of the true underlying distribution deviates from the assumed model~\citep{Huber:AMS:64,Huber:Book:09,Hampel:Book:86}. More specifically, let $F_{\thetav}$ be the cumulative density function (CDF) of the assumed model with $\thetav$ being parameter, then the classic robust inference deals with the situation where the data points $\mathbf{x}_n$ are independently and identically generated by an unknown CDF $F$ in the $\eta$-neighborhood $\mathscr{P}_{\eta}(F_{\thetav})$ of the assumed model $F_{\thetav}$. The goal of the classic robust inference is to design inference algorithms that perform well for any $F\in \mathscr{P}_{\eta}(F_{\thetav}) $. For example, $\mathscr{P}_{\eta}(F_{\thetav})$ can be a L$\acute{\text{e}}$vy neighborhood $\mathscr{P}_{\eta}(F_{\thetav})=\{F|\forall t,\; F_{\thetav}(t-\eta)\leq F(t) \leq F_{\thetav}(t+\eta)+\eta \}$ or contamination neighborhood $\mathscr{P}_{\eta}(F_{\thetav})=\{F| (1-\eta) F_{\thetav}+\eta H\}$ in which $H$ can be any probability measure. The contamination neighborhood model can be viewed as having $\eta$ fraction of the data as outliers, while the L$\acute{\text{e}}$vy neighborhood (or other related neighborhood) is useful in scenarios with model uncertainties. Various concepts such as influence function (IF), breakdown point, and change of variance etc were developed to quantify the robustness of estimators against the presence of outliers. 

In this paper, we consider a setup with more powerful adversaries than those considered in the classic robust inference and investigate \emph{adversarial robustness}. In particular, in the considered setup, after data points $\{\mathbf{x}_1,\cdots,\mathbf{x}_N\}$ are generated, the adversary can observe the whole dataset and then modify all data points to $\{\mathbf{x}_1+\Delta \mathbf{x}_1,\cdots,\mathbf{x}_N+\Delta \mathbf{x}_N\}$ where each $\Delta \mathbf{x}_n$ is carefully designed and depends on the whole data set. It is easy to see that \emph{the adversary in the adversarially robust model is more powerful}. In particular, for any $F\in \mathscr{P}_{\eta}(F_{\thetav})$ in the classic model, the adversary in the adverarially robust model can mimic the behavior of $F$ by simply replacing the dataset with i.i.d samples generated from $F$. Clearly, this is not optimal strategy that the adversary in our model will adopt, as the adversary is not limited to this type of i.i.d attacks after observing the whole dataset. It can construct correlated attack signals that are based on the whole dataset. As the result, it is important to understand the following questions: 1) What is the attacker's optimal attack strategy in choosing $\Delta \mathbf{x}$?; 2) What are the impacts of these attacks?; 3) How shall we design inference algorithms to minimize the impact?

In our recent work~\citep{Lai:TIT:18}, we made some progress in addressing these problems for the case of scalar parameter estimation, in which the parameter to be estimated is a scalar and each sample $x_n$ is also a scalar. In particular, given a data set $\mathbf{x}=\{x_1,\cdots,x_N\}$ with $x_n$ being i.i.d realizations of random variable $X$ that has CDF $F_{\theta}(x)$ with unknown scalar parameter $\theta$, we would like to estimate the unknown parameter $\theta$. There is an adversary who can observe the whole dataset and can modify the dataset to $\mathbf{x}^{\Delta}=\mathbf{x}+\Delta \mathbf{x}:=\{x_1+\Delta x_1, \cdots,x_N+\Delta x_N\}$, in which $\Delta \mathbf{x}=\{\Delta x_1, \cdots,\Delta x_N\}$ is the attack vector chosen by the adversary after observing $\mathbf{x}$. Certain restrictions need to be put on $\Delta \mathbf{x}$, otherwise the estimation problem will not be meaningful. In~\citep{Lai:TIT:18}, we assume that 
$\Delta \mathbf{x}\in \mathscr{P}_{\eta}:=\{\frac{1}{N}||\Delta \mathbf{x}||_p^p\leq \delta^p\},$
in which $||\cdot||_p$ is the $\ell_p$ norm. This type of constraints are reasonable and are motivated by real life examples. For example, in generating adversary examples in images~\citep{Pimentel:ISIT:17,szegedy2013intriguing,easyfooldnn,Goodfellow:ICLR:15,Carlini:ISSP:17}, the total distortion should be limited, otherwise human eyes will be able to detect such changes. 
The classic setup with contamination model~\citep{Huber:AMS:64,Huber:Book:09,Hampel:Book:86} can be viewed as a special case of our formulation by letting $p\rightarrow 0$, i.e., the classic setup has a constraint on the total number of data points that the attacker can modify. For a given estimator, we would like to characterize how sensitive the estimator is with respect to the adversarial attack. In~\citep{Lai:TIT:18}, we considered a scenario where the goal of the attacker is to maximize the estimation error caused by the attack. We introduced a concept named ``adversarial influence function'' (AIF) to quantify the asymptotic rate at which the attacker can introduce estimation error through its optimal attack. From the defender's perspective, the smaller AIF is, the more adversarially robust the estimator is. In~\citep{Lai:TIT:18}, building on the characterization of AIF, we characterized the optimal estimator, among a certain class of estimators, that minimizes AIF. From this characterization, we show that there is a tradeoff between the robustness against adversarial attacks and robustness against outliers. In~\citep{Lai:TIT:18}, we further designed optimal estimator that achieve the optimal tradeoff among these quantities for the scalar case.

In this paper, we extend our work in~\citep{Lai:TIT:18} to multivariate setup, in which the goal is to jointly estimate \emph{multiple} parameters from vector observations. The multivariate setup includes many important cases such as the joint location-scale estimation and robust linear regression etc. In this multivariate setup, we have $N$ data points $\{\mathbf{x}_n, n=1,\cdots,N\}$ with each data point $\mathbf{x}_n\in \mathbb{R}^m$ being a vector. These data points are realizations of a random variable that has CDF $F_{\thetav}(\xv)$ with unknown parameter vector $\thetav\in \mathbb{R}^q$. We use $m\times N$ matrix $\Xv:=[\xv_1,\cdots,\xv_N]$ to denote the given data matrix. From this given data set, we would like to estimate the unknown parameter $\thetav$. The adversary will modify the data to $\mathbf{\Xv}^{\Delta}=\Xv+\Delta \mathbf{X}$, in which $\Delta \Xv$ is the attack matrix chosen by the adversary after observing $\Xv$. Similar to the situation in the classic robust inference problem~\citep{hubert2008}, the multivariate adversarial robustness setup is significantly more challenging than the scalar case.

Firstly, the characterization of the optimal attack strategy is much more difficult. There are many more degrees of freedom for the attacker to choose from, as the dimension of $\Delta \mathbf{X}$ is $m\times N$. Furthermore, each modification will affect all $q$ components of the estimated vector $\hat{\thetav}$ in a different but coupled manner. In this paper, we focus on the class of $M$-estimators specified by $q$-dimension functions $\psiv$. For this class of estimators, we characterize the optimal attack vector $\Delta \mathbf{X}$ and the corresponding AIF. We further simplify this general formula for robust linear regression and evaluate the adversarial robustness of various existing robust algorithms. 

Secondly, the characterization of the optimal defense strategy is also much harder. For example, in the $M$-estimator case, now $\psiv$ is a $q$-dimension function, and the corresponding optimization problem of maximizing AIF becomes a coupled multi-dimension calculus of variation problem, which is in general very challenging. In this paper, for the important case of joint location-scale estimation problem, we show that we can decouple the characterization of optimal defense problem into two scalar problems. Building on this, we identify the optimal $M$-estimator that minimizes AIF.  In addition, similar to the scalar problem, we show that there exist a tradeoff between the robustness against adversarial attack and robustness again outliers. We further characterize the optimal $M$-estimator that achieves the optimal tradeoff between these robustness metrics.


The remainder of the paper is organized as follows. In Section~\ref{sec:intro}, we introduce the considered model. In Section~\ref{sec:general}, we derive general AIF results and simplify the results for robust linear regression problems. In Section~\ref{sec:LSE}, we focus on the special case of joint location-scale estimation problem and characterize the optimal estimators that achieve the optimal AIF for both cases with and without constraints on the robustness against outliers. In Section~\ref{sec:simulation}, we use several numerical examples to illustrate results derived in this paper. Section~\ref{sec:con} provides concluding remarks. 

\section{Model}~\label{sec:intro}

In this section, we first introduce our problem formulation. We will then briefly review results from classic robust estimation that are directly related to our study.


\subsection{Problem Formulation}

We have $N$ data points $\{\mathbf{x}_n, n=1,\cdots,N\}$ with $\mathbf{x}_n:=[\xv_{n,1},\cdots,\xv_{n,m}]^T\in \mathbb{R}^m$. These data points are i.i.d realizations of a random variable that has CDF $F_{\thetav}(\xv)$ with unknown parameter $\thetav\in \Thetav \subset \mathbb{R}^q$. Here, $\Thetav$ is a compact set. We will use $f_{\thetav}(\xv)$ to denote the corresponding probability density function (pdf). We use $m\times N$ matrix $\Xv:=[\xv_1,\cdots,\xv_N]$ to denote the given data matrix. From this given data set, we would like to estimate the unknown parameter $\thetav$. However, as the adversary has access to the whole dataset, it will modify the data to $\mathbf{\Xv}^{\Delta}=\Xv+\Delta \mathbf{X}:=[\xv_1+\Delta \xv_1, \cdots,\xv_N+\Delta \xv_N]$, in which $\Delta \Xv:=[\Delta \xv_1, \cdots,\Delta \xv_N]$ is the attack matrix chosen by the adversary after observing $\Xv$. We will discuss the attacker's optimal attack strategy in choosing $\Delta \Xv$ in the sequel. In this work, we consider the case where the attacker can modify all data points, which is a more suitable setup for recent data analytical applications. However, certain restrictions need to be put on $\Delta \Xv$, otherwise the estimation problem will not be meaningful. In this paper, we assume that 
\begin{eqnarray}
\frac{1}{mN}||\Delta \Xv||_p^p\leq \delta^p,\label{eq:norconstranit}
\end{eqnarray}
in which $||\mathbf{A}||_p$ is the entry-wise matrix $p$-norm:
\begin{eqnarray}
||\mathbf{A}||_p=||\text{Vec}(\mathbf{A})||_p=\left(\sum\sum |\mathbf{A}_{i,j}|^p\right)^{1/p},\nonumber
\end{eqnarray}
with $\text{Vec}(\mathbf{A})$ being the vectorization of matrix $\mathbf{A}$. In~\eqref{eq:norconstranit}, we have the normalization term $\frac{1}{mN}$ as the matrix $\Delta\Xv$ is of size $m\times N$. The normalization factor implies that the per-dimension change (on average) is upper-bound by $\delta^p$. As mentioned in the introduction, this type of constraints are reasonable and are motivated by real life examples. 

Following notation used in robust statistics \citep{Huber:Book:09,Hampel:Book:86}, we will use $\tv_N(\Xv)$, a $q$ dimensional vector, to denote an estimator. For a given estimator $\tv_N$, we would like to characterize how sensitive the estimator is with respect to the adversarial attack. In this paper, we consider a scenario where the goal of the attacker is to maximize the estimation error caused by the attack. In particular, the attacker aims to choose $\Delta\Xv$ by solving the following optimization problem
\begin{eqnarray}\label{eq:optimalintro}
\max\limits_{\Delta\Xv} && ||\tv_N(\Xv+\Delta\Xv)-\tv_N(\Xv)||_1,\\
\text{s.t.}&& \frac{1}{mN}||\Delta \Xv||_p^p\leq \delta^p,\nonumber
\end{eqnarray}
in which $||\cdot||_1$ is the $\ell_1$ norm.

We use $f(\tv_N,\Xv,p,\delta)$ to denote the optimal value obtained from the optimization problem~\eqref{eq:optimalintro}, and define the adversarial influence function (AIF) of estimator $\tv_N$ at $\Xv$ under $\ell_p$ norm constraint as
\begin{eqnarray}
\text{AIF}(\tv_N,\Xv,p)=\lim_{\delta\downarrow 0}\frac{f(\tv_N,\Xv,p,\delta)}{\delta}.\nonumber
\end{eqnarray}
This quantity, a generalization of the concept of influence function (IF) used in classic robust estimation (a brief review of IF will be provided in Section~\ref{sec:IF}), quantifies the asymptotic rate at which the attacker can introduce estimation error through its attack. 

From the defender's perspective, the smaller AIF is, the more robust the estimator is. In this paper, building on the characterization of $\text{AIF}(\tv_N,\Xv,p)$, we will characterize the optimal estimator $\tv_N$, among a certain class of estimators $\mathcal{T}$, that minimizes $\text{AIF}(\tv_N,\Xv,p)$. In particular, we will investigate 
\begin{eqnarray}
\min\limits_{\tv_N\in \mathcal{T}} \text{AIF}(\tv_N,\Xv,p).\nonumber
\end{eqnarray}

Note that $\text{AIF}(\tv_N,\mathbf{X},p)$ depends on the data matrix $\mathbf{X}$. Based on the characterization of AIF for a given data realization matrix $\mathbf{X}$ with $N$ columns (each column representing one data point), we will then study the population version of AIF where each column of $\mathbf{X}$ is i.i.d generated by $F_{\thetav}$. We will examine the behavior of $\text{AIF}(\tv_N,\mathbf{X},p)$ as $N$ increases. We will see that for a large class of estimators $\text{AIF}(\tv_N,\mathbf{X},p)$ has a well-defined limit as $N\rightarrow \infty$. We will use $\text{AIF}(\tv,F_{\thetav},p)$ to denote this limit when it exists.


From the defense's perspective, we would like to design an estimator that is least sensitive to the adversarial attack. Again, we will characterize the optimal estimator $\tv$, among a certain class of estimators $\mathcal{T}$, that minimizes $\text{AIF}(\tv,F_{\thetav},p)$. That is, for a certain class of estimators $\mathcal{T}$, we will solve
\begin{eqnarray}
\min\limits_{\tv\in \mathcal{T}}  \text{AIF}(\tv,F_{\thetav},p).\label{eq:AIF}
\end{eqnarray}

It will be clear in the sequel that the solution to the optimization problem~\eqref{eq:AIF}, even though is robust against adversarial attacks, has poor performance in guarding against outliers. This motivates us to design estimators that strike a desirable tradeoff between these two robustness measures. In particular, we will solve~\eqref{eq:AIF} with an additional constraint on IF. After the corresponding quantities are introduced in later sections, precise formulation of this optimization problem with additional IF constraint will be given in~\eqref{eq:AIFcon}. 

\subsection{M-Estimator and Influence Function (IF)}\label{sec:IF}

In this paper, we will mainly focus on a class of commonly used estimator in robust statistic: $M$-estimator~\citep{Huber:AMS:64}, in which one obtains an estimate $\tv_N(\Xv)$ of $\thetav$ by solving 
\begin{eqnarray}\label{eq:Mestimategen}
\sum\limits_{n=1}^N \psiv(\xv_n,\tv_N)=\mathbf{0}.
\end{eqnarray}

Here $\psiv(\xv,\thetav):\mathbb{R}^m\times \mathbb{R}^q\rightarrow \mathbb{R}^q$ is a vector function of data point $\xv$ and parameter $\thetav$ to be estimated. We use $\psi_i$, $i=1,\cdots,q$, to denote each component of $\psiv$. Different choices of $\psiv$ lead to different robust estimators. For example, the most likely estimator (MLE) can be obtained by setting $\psiv=-f_{\thetav}^{'}/f_{\thetav}$. 

As the form of $\psiv$ determines $\tv_N$, in the remainder of the paper, we will use $\psiv$ and $\tv_N$ interchangeably. For example, we will denote $\text{AIF}(\tv_N,\Xv,p)$ as $\text{AIF}(\psiv,\Xv,p)$. 

It is typically assumed that $\psiv(\xv,\thetav)$ is continuous and almost everywhere differentiable. This assumption is valid for all $\psiv$'s that are commonly used. It is also typically required that the estimator is Fisher consistent~\citep{Hampel:Book:86}:
\begin{eqnarray}
\mathbb{E}_{F_{\thetav}}[\psiv(\Xv,\thetav)]=\mathbf{0},\label{eq:Fisher}
\end{eqnarray}
in which $\mathbb{E}_{F_{\thetav}}$ means expectation under $F_{\thetav}$. Intuitively speaking, this implies that the true parameter $\thetav$ is the solution of the $M$-estimator if there are increasingly more i.i.d. data points generated from $F_{\thetav}$.

In the contamination model of the classic robust estimation setup, it is assumed that a fraction $\delta$ of data points are outliers, while the remainder of data points are generated from the true distribution $F_{\thetav}$. For a given estimator $\tv$, the concept of IF introduced by Hamper~\citep{Hampel:PHD:68} is defined
\begin{eqnarray}
\text{IF}(\xv,\tv,F_{\thetav})=\lim_{\delta\downarrow 0}\frac{\tv((1-\delta)F_{\thetav}+\delta i_{\xv})-\tv(F_{\thetav})}{\delta}.\nonumber
\end{eqnarray}
Here, $\text{IF}(\xv,\tv,F_{\thetav})$ is a $q$ dimensional vector. In this definition, $i_{\xv}$ is a distribution that puts mass 1 at point $\xv$. In addition, $\tv(F_{\thetav})$ is the obtained estimate when all data points are generated i.i.d from $F_{\thetav}$, and $\tv((1-\delta)F_{\thetav}+\delta i_{\xv})$ is the obtained estimate when $1-\delta$ fraction of data points are generated i.i.d from $F_{\thetav}$ while $\delta$ fraction of the data points are at $\xv$. Hence, $\text{IF}(\xv,\tv,F_{\thetav})$ measures the asymptotic influence of having outliers at point $\xv$ as $\delta\downarrow 0$. Similar as above, as $\tv$ is determined by $\psiv$ in M-estimator, in the following, we will also denote $\text{IF}(\xv,\tv,F_{\thetav})$ as $\text{IF}(\xv,\psiv,F_{\thetav})$.

Furthermore, to characterize the impact of the worst outliers, Hamper~\citep{Hampel:Book:86} introduced the (unstandardized) gross-error sensitivity:
\begin{eqnarray}
\gamma^*_u(\psiv,F_{\thetav})=\sup\limits_{\xv}\{||\text{IF}(\xv,\psiv,F_{\thetav})||_2\},\label{eq:gross}
\end{eqnarray}
in which $||\cdot||_2$ is the $\ell_2$ norm.

For $M$-estimator, $\text{IF}(\xv,\psiv, F_{\thetav})$ was shown to be~\citep{Hampel:Book:86} 
\begin{eqnarray}
\text{IF}(\xv,\psiv, F_{\thetav})=\Mv(\psiv,F_{\thetav})^{-1}\psiv(\xv,\tv(F_{\thetav})),\no
\end{eqnarray}
with the $q\times q$ matrix $\Mv$ given by
\begin{eqnarray}
\Mv(\psiv,F_{\thetav})=-\int\left.\frac{\partial}{\partial\thetav}\psiv(\mathbf{x},\thetav)\right\vert_{\tv(F_{\thetav})}\text{d}F_{\thetav}(\xv),\no
\end{eqnarray}
see (4.2.9) of~\citep{Hampel:Book:86}. 

\section{Characterizing AIF}\label{sec:general}

In this section, for a given data matrix $\Xv$, we analyze the AIF for any given M-estimator $\psiv$ as specified in~\eqref{eq:Mestimategen}. As $\psiv$ is $q$-dimension vector, there are $q$ equations. To simplify the presentation, we write each equation as $$G_i(\Xv,\tv_N):=\sum\limits_{n=1}^N \psi_i(\xv_n,\tv_N)=0,\hspace{3mm}i=1,\cdots,q$$ and denote $\mathbf{G}=[G_1,\cdots,G_q]^T$. Using this notation,~\eqref{eq:Mestimategen} can be written as
\begin{eqnarray}
\mathbf{G}(\Xv,\tv_N)=\mathbf{0}.\label{eq:tveq}
\end{eqnarray}

\subsection{General Case}\label{sec:aifgeneral}

To proceed further, we write
\begin{eqnarray}
\mathbf{G}^{'}_{\thetav}(\Xv,\tv_N)&=&\left.\frac{\partial \mathbf{G}}{\partial \thetav}\right\vert_{\Xv,\tv_N},\no\\
\mathbf{G}^{'}_{\Xv}(\Xv,\tv_N)&=&\left.\frac{\partial {\mathbf{G}}}{\partial \text{Vec}(\mathbf{X})}\right\vert_{\Xv,\tv_N},\no
\end{eqnarray}
and
\begin{eqnarray}
\mathbf{\tv}^{'}_{N}(\Xv)=\left.\frac{\partial \mathbf{t}_N}{\partial \text{Vec}(\mathbf{X})}\right\vert_{\Xv,\tv_N}.\no
\end{eqnarray}

We have the following characterization.

\begin{thm}\label{thm:AIFgen}
	For $p\geq 1$, suppose $\mathbf{G}^{'}_{\thetav}(\Xv,\tv_N)$ is invertible, we have
	\begin{eqnarray}
	\AIF(\psiv,\Xv,p)=(mN)^{1/p}\max\limits_{\sigmav\in\Sigmav}\left|\left|\sigmav^T\left[\mathbf{G}^{'}_{\thetav}(\Xv,\tv_N)\right]^{-1}\mathbf{G}^{'}_{\Xv}(\Xv,\tv_N)\right|\right|_{\frac{p}{p-1}},\label{eq:AIFg}\end{eqnarray}
	in which $\Sigmav$ is the set of length $q$ vectors with each entry being either $1$ or $-1$.
\end{thm}
\begin{proof}
	First, from~\eqref{eq:tveq}, we have 
	\begin{eqnarray}
	\tv_N^{'}(\Xv)=-\left[\mathbf{G}^{'}_{\thetav}(\Xv,\tv_N)\right]^{-1}\mathbf{G}^{'}_{\Xv}(\Xv,\tv_N).\label{eq:tvderiv}
	\end{eqnarray}
	
	Using Taylor expansion, we have
	\begin{eqnarray}
	\tv_N(\Xv+\Delta\Xv)-\tv_N(\Xv)=\tv_N^{'}(\Xv)\text{Vec}(\Delta\Xv)+\text{higher order terms}.\no
	\end{eqnarray}
	
	When $\delta$ is small, the adversary can focus on the following optimization problem to obtain an $o(\delta)$ optimal solution
	\begin{eqnarray}
	\text{(P1):}\hspace{10mm}	\max\limits_{\text{Vec}(\Delta \Xv)}&&||\tv_N^{'}(\Xv)\text{Vec}(\Delta\Xv)||_1,\label{eq:original}\\\nonumber
	\text{s.t.}&& ||\Delta\Xv||_p^p\leq Nm\delta^p.
	\end{eqnarray}
	
	Let $f^*$ be the optimal value obtained in the optimization problem (P1). Let $g^*$ be the optimal value of the following optimization problem
	\begin{eqnarray}
	\text{(P2):}\hspace{10mm}	\max\limits_{\sigmav\in\Sigmav}\max\limits_{\text{Vec}(\Delta \Xv)}&&\sigmav^T\tv_N^{'}(\Xv)\text{Vec}(\Delta\Xv),\no\\\nonumber
	\text{s.t.}&& ||\Delta\Xv||_p^p\leq Nm\delta^p.
	\end{eqnarray}
	
	In Appendix~\ref{app:fg}, we show that $f^*=g^*$. Hence, we can focus on problem (P2). 
	
	The inner maximization problem in (P2) is the same as
	\begin{eqnarray}
	\min\limits_{\Delta \Xv}&&-\sigmav^T\tv_N^{'}(\Xv)\text{Vec}(\Delta\Xv),\no\\\nonumber
	\text{s.t.}&& ||\Delta\Xv||_p^p\leq Nm\delta^p.
	\end{eqnarray}
	
	Using~\eqref{eq:tvderiv}, we have 
	\begin{eqnarray}
	-\sigmav^T\tv_N^{'}(\Xv)\text{Vec}(\Delta\Xv)=\sigmav^T\left[\mathbf{G}^{'}_{\thetav}(\Xv,\tv_N)\right]^{-1}\mathbf{G}^{'}_{\Xv}(\Xv,\tv_N)\text{Vec}(\Delta\Xv).\no
	\end{eqnarray}	
	
	To simplify the notation, we denote
	\begin{eqnarray}
	\mathbf{a}=\sigmav^T\left[\mathbf{G}^{'}_{\thetav}(\Xv,\tv_N)\right]^{-1}\mathbf{G}^{'}_{\Xv}(\Xv,\tv_N),\no
	\end{eqnarray}
	which is a row vector with $mN$ entries. Even though $\mathbf{a}$ is only a row vector, we denote these elements as $a_{i,j}$ for $i=1,\cdots,m$ and $j=1,\cdots,N$ to better connect with each elements of $\Delta\Xv$. Hence, $a_{i,j}$ corresponds to $\Delta \Xv_{i,j}$. Using this notation, the optimization problem can be written as 
	\begin{eqnarray}\label{eq:lineargen}
	\min\limits_{\Delta\Xv}&& \mathbf{a}\text{Vec}(\Delta \Xv),\nonumber\\
	\text{s.t.}&& ||\Delta \Xv||_p^p\leq mN\delta^p.\no
	\end{eqnarray}
	
	For $p=1$, this is a linear programing problem, whose solution is simple. In particular, let $a^*=\max\limits_{i,j} | a_{i,j}|$, and $(i^*,j^*)$ be the corresponding index, it is easy to check that we have
	\begin{eqnarray}
	\Delta \Xv_{i^*,j^*}^*=-\text{sign}\left\{ a_{i^*,j^*}\right\}mN\delta,\nonumber
	\end{eqnarray}
	and $\Delta \Xv_{i,j}^*=0$ for other $i,j$s. Hence,
	\begin{eqnarray}
	\AIF(\psiv,\Xv,p)=\max\limits_{\sigmav\in\Sigmav}mNa^*=mN\max\limits_{\sigmav\in\Sigmav}||\mathbf{a}||_{\infty}.\nonumber
	\end{eqnarray}

	For $p> 1$,~\eqref{eq:lineargen} is a convex optimization problem. To solve this, we form Lagrange
	\begin{eqnarray}
	\mathcal{L}(\Delta \Xv,\lambda)= \mathbf{a}\text{Vec}(\Delta \Xv)+ \lambda\left(||\Delta \Xv||_p^p-mN\delta^p\right).\nonumber
	\end{eqnarray}
	
	The corresponding optimality conditions are:
	\begin{eqnarray}
	a_{i,j}+\lambda^* p \text{sign} (\Delta \Xv_{i,j}^*) |\Delta \Xv_{i,j}^*|^{p-1}&=&0, \hspace{2mm} \forall n\label{eq:laggen}\\
	\lambda^*&\geq&0,\nonumber\\
	\lambda^*\left(||\Delta \Xv||_p^p-mN\delta^p\right)&=&0.\nonumber
	\end{eqnarray}
	
	From~\eqref{eq:laggen}, we know that $\lambda^*\neq 0$, hence 
	\begin{eqnarray}
	||\Delta \Xv^*||_p^p=mN\delta^p,\label{eq:equalgen}
	\end{eqnarray}
	and
	\begin{eqnarray}
	\text{sign} (\Delta \Xv_{i,j}^*) |\Delta \Xv_{i,j}^*|^{p-1}=\frac{-a_{i,j}}{\lambda^* p}.\label{eq:signgen}
	\end{eqnarray}
	
	From~\eqref{eq:signgen} and the fact that $\lambda^* p$ is positive, we know $\text{sign} (\Delta \Xv_{i,j}^*)=-\text{sign}(a_{i,j})$, and hence we have
	\begin{eqnarray}
	|\Delta \Xv_{i,j}^*|^{p-1}=\frac{|a_{i,j}|}{\lambda^* p},\nonumber
	\end{eqnarray}
	which can be simplified further to
	\begin{eqnarray}
	\Delta \Xv_{i,j}^*=-\left(\frac{|a_{i,j}|}{\lambda^* p}\right)^{1/(p-1)} \text{sign} (a_{i,j}).\nonumber
	\end{eqnarray}
	
	Combining these with~\eqref{eq:equalgen}, we obtain the value of $\lambda^*$:
	\begin{eqnarray}
	\lambda^*=\frac{1}{p}\left(\frac{\sum\limits_{n=1}^N |a_{i,j}|^{p/(p-1)}}{mN\delta^p}\right)^{(p-1)/p}.\nonumber
	\end{eqnarray}
	
	As the result, we have
	\begin{eqnarray}
	\Delta \Xv_{i,j}^*=-\frac{|a_{i,j}|^{1/(p-1)}(mN)^{1/p}}{(\sum|a_{i,j}|^{p/(p-1)})^{1/p}}\text{sign}(a_{i,j})\delta.\nonumber
	\end{eqnarray}
	
	Hence, the optimal value of the inner maximization of (P2) is 
	\begin{eqnarray}
	\sum\limits_{i,j} a_{i,j}\frac{|a_{i,j}|^{1/(p-1)}(mN)^{1/p}}{(\sum|a_{i,j}|^{p/(p-1)})^{1/p}}\text{sign}(a_{i,j})	=(mN)^{1/p}||\mathbf{a}||_{p/(p-1)},\nonumber
	\end{eqnarray}
	which finishes the proof.
\end{proof}
By setting $m=1$ and $q=1$, \eqref{eq:AIFg} recovers the result on the scalar case presented in~\citep{Lai:TIT:18}.

\subsection{Robust Regression}\label{sec:linearregression}

In this section, we use robust linear regression, an important multivariate parameter estimation problem, as an example to illustrate the result derived in Section~\ref{sec:aifgeneral}. In linear regression problems, the data points are $\tilde{\xv}_n=\left(\begin{array}{c}
\mathbf{x}_n\\
y_n
\end{array}\right)$, $n=1,\cdots,N$ with $\mathbf{x}_n\in \mathbb{R}^q$ and $y_n\in \mathbb{R}$. Hence, $\tilde{\xv}_n\in \mathbb{R}^{q+1}$. In the following, we let $\tilde{\Xv}=[\tilde{\xv}_1,\cdots,\tilde{\xv}_N]\in \mathbb{R}^{(q+1)\times N}$, and still denote $\Xv=[\xv_1,\cdots,\xv_N]\in \mathbb{R}^{q\times N}$. From the data, we would like to fit a linear model, i.e., we would like to find $\thetav=[\theta_1,\cdots,\theta_q]^T\in \mathbb{R}^q$ such that $\mathbf{x}_n^T\thetav$ is a good approximation of $y_n$. Hence, the parameters to be estimated are $\thetav\in \mathbb{R}^q$. Furthermore, each data point $\tilde{\xv}_n\in \mathbb{R}^{q+1}$, hence $m=q+1$. We denote \begin{eqnarray}
r_n=y_n-\mathbf{x}_n^T\thetav\no
\end{eqnarray} as the residual error.  

The commonly used ordinary least square (OLS) approach finds $\thetav$ by solving
\begin{eqnarray}
\min\limits_{\thetav} \sum\limits_{n=1}^N r_n^2,\no
\end{eqnarray}
which is equivalent to solving
\begin{eqnarray}
\sum\limits_{n=1}^N r_n\mathbf{x}_n=\mathbf{0}.\label{eq:OLS}
\end{eqnarray}
The solution is well known $
\hat{\thetav}_{OLS}=(\Xv\Xv^T)^{-1}\Xv\mathbf{y}.
$ In the subsequent discussion, we will use a related quantity named hat matrix \begin{eqnarray}\Hv=\Xv^T(\Xv\Xv^T)^{-1}\Xv.\label{eq:Hat}\end{eqnarray}

It is known that OLS solution is not robust to outliers~\citep{Hampel:Book:86}. Various robust linear regression schemes were proposed~\citep{Wilcox:Book:05,Mallows:BTL:75,Huber:AoS:73,Merrill:TP:71}. They generally set $\psiv(\tilde{\xv}_n,\tv_N)$ in the form $\eta(r_n v_n)w_n\xv_n$ 
with function $\eta: \mathbb{R}\rightarrow \mathbb{R}$ and weights $w_n$ and $v_n$. That is, for robust linear regression, one obtains the estimate $\tv_N$ of $\thetav$ by solving
\begin{eqnarray}\label{eq:Mestimate}
\sum\limits_{n=1}^N\psiv(\tilde{\xv}_n,\tv_N)=\sum\limits_{n=1}^N \eta(r_n v_n)w_n\xv_n=\mathbf{0}.
\end{eqnarray}

The weights $w_n$ and $v_n$ can be chosen to not only depend on $\xv_n$ but also the whole data matrix $\Xv$. For example, it is common~\citep{Wilcox:Book:05,Mallows:BTL:75} to use $w_n=\sqrt{1-h_{nn}}$, in which $h_{nn}=\xv_n^T(\Xv\Xv^T)^{-1}\xv_n$ is the $n$th diagonal element of the hat matrix~\eqref{eq:Hat}. It is known~\citep{Huber:Book:09} that $0\leq h_{nn}\leq 1$. Comparing~\eqref{eq:Mestimate} with~\eqref{eq:OLS}, we can see that one replaces $\xv_n$ in~\eqref{eq:OLS} with $w_n\xv_n$ and replaces $r_n $ in~\eqref{eq:OLS} with $\eta(r_n v_n)$. The main idea is to use $w_n$ to limit the impact of outliers in $\xv_n$ and use $\eta(r_n v_n)$ to limit the impact of outliers in the residual $r_n$ while taking the location of $\xv_n$ into consideration. From~\eqref{eq:Mestimate}, we have 
\begin{eqnarray}
\psi_i(\tilde{\xv}_n,\tv_N)&=&\eta(r_n v_n)w_n\xv_{n,i}=\eta((y_n-\mathbf{x}_n^T\tv_N)v_n)w_n\xv_{n,i},\no\\
G_i(\tXv,\tv_N)&=&\sum\limits_{n=1}^N\psi_i(\tilde{\xv}_n,\tv_N)=\sum\limits_{n=1}^N\eta((y_n-\mathbf{x}_n^T\tv_N)v_n)w_n\xv_{n,i}.\no
\end{eqnarray}

Different choices of functions lead to different classes of robust linear regression methods. For example:
\begin{itemize}
	\item $w_n=1$, $v_n=1$ leads to Huber's proposal~\citep{Huber:AoS:73}, in which the idea was to replace $r_n$ in~\eqref{eq:OLS} with $\eta(r_n)$ so as to limit the influence of large residuals (similar to the M-estimator in a single variable case). 
	\item $v_n=1$ leads to Mallows's proposal.
	\item $v_n=1/w_n$ is Schewppe's approach~\citep{Merrill:TP:71}. 
\end{itemize}

In the following, we calculate $\mathbf{G}^{'}_{\tilde{\Xv}}(\tilde{\Xv},\tv_N)$ and $\mathbf{G}^{'}_{\thetav}(\tXv,\tv_N)$.

First, we compute $\mathbf{G}^{'}_{\tilde{\Xv}}(\tilde{\Xv},\tv_N)$. For $i=1,\cdots,q$; $j=1,\cdots,N$; and $k=1,\cdots,q$, we have
\begin{eqnarray}
\left.\frac{\partial G_i}{\partial \mathbf{x}_{j,k}}\right\vert_{\tilde{\Xv},\tv_N}&=&\frac{\partial \left(\sum\limits_{n=1}^N w_n\eta(r_n v_n)\mathbf{x}_{n,i}\right)}{\partial \mathbf{x}_{j,k}}\nonumber\\
&=&\sum\limits_{n=1}^N \left[ \eta(r_n v_n) \mathbf{x}_{n,i}\frac{\partial  w_n}{\partial \mathbf{x}_{j,k}}+w_n \mathbf{x}_{n,i}\eta^{'}(r_n v_n)\left(-\tv_{N}(k)v_n\mathbb{I}(j=n)+r_n\frac{\partial v_n }{\partial \mathbf{x}_{j,k}}\right)\right.\nonumber\\
&&\left.+w_n\eta(r_n v_n)\mathbb{I}(j=n,k=i)\right], \no
\end{eqnarray}
in which $\mathbb{I}(\cdot)$ is the indicator function and $\tv_{N}(k)$ is the $k$th element of $\tv$. In addition,
\begin{eqnarray}
\left.\frac{\partial G_i}{\partial y_{j}}\right\vert_{\tilde{\Xv},\tv_N}=w_j\eta^{'}(r_jv_j)v_j\mathbf{x}_{j,i}:=c_j\mathbf{x}_{j,i},\no
\end{eqnarray}
in which we denote \begin{eqnarray}c_j=w_j\eta^{'}(r_j v_j)v_j.\label{eq:cj}\end{eqnarray}

Furthermore, each entry of $\mathbf{G}^{'}_{\thetav}(\tXv,\tv_N)$ can be computed as
\begin{eqnarray}
\left.\frac{\partial G_i}{\partial \theta_{j}}\right\vert_{\tilde{\Xv},\tv_N}=-\sum\limits_{n=1}^N w_n\eta^{'}(r_n v_n)v_n\mathbf{x}_{n,i}\mathbf{x}_{n,j}:=-\sum\limits_{n=1}^N c_n\mathbf{x}_{n,i}\mathbf{x}_{n,j}.	\no
\end{eqnarray}
From this, we know that
\begin{eqnarray}
\mathbf{G}^{'}_{\thetav}(\tXv,\tv_N)=-\Xv\text{diag}[c_1,\cdots,c_N]\Xv^T.\no
\end{eqnarray}

Using the result in Theorem~\ref{thm:AIFgen}, assuming $\mathbf{G}^{'}_{\thetav}(\tXv,\tv_N)$ is invertible, we have the following characterization of AIF of robust linear regression.
\begin{prop}\label{prop:AIF}
	For robust linear regression,
	\begin{eqnarray}
	\AIF(\psiv,\tXv,p)=(mN)^{1/p}\max\limits_{\sigmav\in\Sigmav}\left|\left|\sigmav^T\left[\Xv\text{diag}[c_1,\cdots,c_N]\Xv^T\right]^{-1}\mathbf{G}^{'}_{\tilde{\Xv}}(\tilde{\Xv},\tv_N)\right|\right|_{\frac{p}{p-1}},\no\end{eqnarray}
	with $c_j$ defined in~\eqref{eq:cj}.
\end{prop}

We now apply Proposition~\ref{prop:AIF} to various specific (robust) linear regression approaches.

\begin{exmpl}\textbf{OLS:} For OLS, we have $w_n=1$, $v_n=1$ and $\eta(r)=r$. In this case, we have
\begin{eqnarray}
\left.\frac{\partial G_i}{\partial \mathbf{x}_{j,k}}\right\vert_{\tilde{\Xv},\tv_N}=-\mathbf{x}_{j,i}\tv_{N}(k)+r_j\mathbb{I}(k=i), \no
\end{eqnarray}
and
$
\left.\frac{\partial G_i}{\partial y_{j}}\right\vert_{\tilde{\Xv},\tv_N}=\mathbf{x}_{j,i}.
$
Furthermore, $c_j=1$ and hence
$
\mathbf{G}^{'}_{\thetav}(\Xv,\tv_N)=-\Xv\Xv^T.
$
\end{exmpl}
\begin{exmpl}\textbf{Huber's proposal~\citep{Huber:AoS:73}:} In Huber's proposal~\citep{Huber:AoS:73}, we have $w_n=1$, $v_n=1$. In this case, we have
\begin{eqnarray}
\left.\frac{\partial G_i}{\partial \mathbf{x}_{j,k}}\right\vert_{\tilde{\Xv},\tv_N}=-\mathbf{x}_{j,i}\tv_{N}(k)\eta^{'}(r_j)+\eta(r_j)\mathbb{I}(k=i), \no
\end{eqnarray}
and
$\left.\frac{\partial G_i}{\partial y_{j}}\right\vert_{\tilde{\Xv},\tv_N}=\eta^{'}(r_j)\mathbf{x}_{j,i}.$ Furthermore
$
\mathbf{G}^{'}_{\thetav}(\Xv,\tv_N)=-\Xv\text{diag}[\eta^{'}(r_1),\cdots,\eta^{'}(r_N)]\Xv^T.
$
\end{exmpl}

\begin{exmpl}\textbf{Mallow's proposal:} In Mallow's proposal, we have $v_n=1$. In this case, we have
\begin{eqnarray}
\left.\frac{\partial G_i}{\partial \mathbf{x}_{j,k}}\right\vert_{\tilde{\Xv},\tv_N}
&=&\sum\limits_{n=1}^N \left[ \eta(r_n ) \mathbf{x}_{n,i}\frac{\partial  w_n}{\partial \mathbf{x}_{j,k}}\right]-w_j\mathbf{x}_{j,i}\tv_{N}(k)\eta^{'}(r_j)+w_j\eta(r_j)\mathbb{I}(k=i),\no
\end{eqnarray}
and
$\left.\frac{\partial G_i}{\partial y_{j}}\right\vert_{\tilde{\Xv},\tv_N}=w_j\eta^{'}(r_j)\mathbf{x}_{j,i}.$ Furthermore,
$
\mathbf{G}^{'}_{\thetav}(\Xv,\tv_N)=-\Xv\text{diag}[c_1,\cdots,c_N]\Xv^T
$
with
$c_j=w_j\eta^{'}(r_j)$.
\end{exmpl}
\begin{exmpl}\textbf{Schewppe's approach~\citep{Merrill:TP:71}:} In Schewppe's approach, we have $v_n=1/w_n$. In this case, we have
\begin{eqnarray}
\left.\frac{\partial G_i}{\partial \mathbf{x}_{j,k}}\right\vert_{\tilde{\Xv},\tv_N}&=&\sum\limits_{n=1}^N \left[ \eta(r_n v_n) \mathbf{x}_{n,i}\frac{\partial  w_n}{\partial \mathbf{x}_{j,k}}-w_n \mathbf{x}_{n,i}\eta^{'}(r_n v_n)\left(\tv_{N}(k)v_n\mathbb{I}(j=n)+\frac{r_n}{w_n^2}\frac{\partial w_n }{\partial \mathbf{x}_{j,k}}\right)\right.\no\\
&&\left.+w_n\eta(r_n v_n)\mathbb{I}(j=n,k=i)\right]. \nonumber
\end{eqnarray}
and
$
\left.\frac{\partial G_i}{\partial y_{j}}\right\vert_{\tilde{\Xv},\tv_N}=\eta^{'}(r_jv_j)\mathbf{x}_{j,i}.
$ Furthermore, $c_j=\eta^{'}(r_jv_j)$ and
$
\mathbf{G}^{'}_{\thetav}(\Xv,\tv_N)=-\Xv\text{diag}[c_1,\cdots,c_N]\Xv^T.
$
\end{exmpl}
We will compare these methods numerically in Section~\ref{sec:simulation}.

\section{Optimal Adversarial Robustness vs Outlier Robustness Tradeoff}\label{sec:LSE}

In this section, we specialize the results to the joint estimation of location and scale. Building on these results, we will design $M$-estimators that minimizes AIF or achieves the optimal tradeoff between AIF, i.e., adversarial robustness, and IF, i.e., outlier robustness.

In the joint location-scale estimation, given $\{x_n,n=1,\cdots,N\}$ with $x_n\in \mathbb{R}$, 
the goal is to jointly estimate the location parameter $\theta_1$ and the scale parameter $\theta_2$. We will that assume $\theta_1$ is bounded and there is a constant $c$ such that $|\theta_2|>c$. Hence, the dimension of each data point $m=1$ and the dimension of parameter $q=2$. We focus on a large class of model named the location-scale model~\citep{Hampel:Book:86}. In the location-scale model, we have 
\begin{eqnarray}
F_{\thetav}(X)=\mathscr{L}(\theta_1+\theta_2Z),\no
\end{eqnarray}
in which $Z$ is a random variable with symmetric pdf $f_0$ and CDF $F_0$, and $\mathscr{L}(\theta_1+\theta_2Z)$ means the distribution of $\theta_1+\theta_2Z$. This class of model includes many important models in statistics. For example the joint estimation of mean and variance of Gaussian random variable belongs to this model. Another important example is the linear regression discussed in Section~\ref{sec:linearregression}. 

Following the convention, we use $T_N$ to denote the estimation of location $\theta_1$ and $S_N$ to denote the estimation of the scale $\theta_2$. In the location-scale model, one typically obtains $(T_N,S_N)$ by solving the following equations~\citep{Huber:Book:09}:
\begin{eqnarray}
\sum\limits_{n=1}^N\psi_1\left(\frac{x_n-T_N}{S_N}\right)&=&0,\no\\
\sum\limits_{n=1}^N\psi_2\left(\frac{x_n-T_N}{S_N}\right)&=&0,\no
\end{eqnarray}
with properly chosen $\psi_1$ and $\psi_2$. In the following, for presentation and notation convenience, we denote $$z_n:=\frac{x_n-T_N}{S_N}.$$
Hence, \begin{eqnarray}
\psiv(x,\thetav)=\left(\begin{array}{cc}
\psi_1\left(z\right)\\
\psi_2\left(z\right)
\end{array}\right), \no
\end{eqnarray} 
with $z=\frac{x-\theta_1}{\theta_2}$.

For this class of model, one typically focuses on equivariant $M$-estimator~\citep{Hampel:Book:86} with \begin{eqnarray}
\psi_1(-z)=-\psi_1(z) \text{ and } \psi_2(-z)=\psi_2(z).\no
\end{eqnarray}
This implies that the first component of $\psiv$ is an odd function, while the second component of $\psiv$ is an even function. Furthermore, $\psi_{1}$ and $\psi_{2}$ are assumed to be monotone functions in $z\geq 0$. Without loss of generality, we will focus on monotone increasing functions, hence $\psi_{1}^{'}(z)\geq 0$ and $\psi_{2}^{'}(z)\geq 0$ for $z\geq 0$. 

\subsection{Given Sample Case}
We first use the results derived in Section~\ref{sec:general} to characterize the AIF for a given data matrix $\Xv$. We note that in this joint location-scale estimation problem, $m=1$, $q=2$, and
\begin{eqnarray}
G_i=\sum\limits_{n=1}^N\psi_i(z_n).\no
\end{eqnarray}

We have
\begin{eqnarray}
\left.\frac{\partial G_i}{\partial x_n}\right\vert_{\Xv,\tv_N}&=&\frac{1}{S_N}\psi_i^{'}(z_n),\hspace{4mm}\left.\frac{\partial G_i}{\partial \theta_1}\right\vert_{\Xv,\tv_N}=-\frac{1}{S_N}\sum\limits_{n=1}^N\psi_{i}^{'}(z_n),\no\\
\left.\frac{\partial G_i}{\partial \theta_2}\right\vert_{\Xv,\tv_N}&=&-\frac{1}{S_N}\sum\limits_{n=1}^Nz_n\psi_{i}^{'}(z_n).\no
\end{eqnarray}

Hence,
\begin{eqnarray}
\mathbf{G}^{'}_{\Xv}(\Xv,\tv_N)&=&\frac{1}{S_N}\left[\begin{array}{ccc}\psi_{1}^{'}(z_1),&,\cdots,&\psi_{1}^{'}(z_N)\\
\psi_{2}^{'}(z_1),&\cdots,&\psi_{2}^{'}(z_N)
\end{array}\right],\no\\
\mathbf{G}^{'}_{\thetav}(\Xv,\tv_N)&=&-\frac{1}{S_N}\left[\begin{array}{cc}\sum\limits_{n=1}^N\psi_{1}^{'}(z_n),&\sum\limits_{n=1}^Nz_n\psi_{1}^{'}(z_n)\\
\sum\limits_{n=1}^N\psi_{2}^{'}(z_n),&\sum\limits_{n=1}^Nz_n\psi_{2}^{'}(z_n)
\end{array}\right]:=-\frac{1}{S_N}\left[\begin{array}{cc}a,&b\\
c,&d
\end{array}\right].\label{eq:abcd}
\end{eqnarray}
Since, $q=2$, $\Sigmav=\{(1,1)^T,(-1,-1)^T,(1,-1)^T,(-1,1)^T\}$. Due to symmetry, we only need to consider $\sigmav$ being either $(1,1)^T$ or $(1,-1)^T$. When $\sigmav\in\{(1,1)^T,(1,-1)^T\}$, we have
\begin{eqnarray}
&&\hspace{-8mm}\sigmav^T\left[\mathbf{G}^{'}_{\thetav}(\Xv,\tv_N)\right]^{-1}\mathbf{G}^{'}_{\Xv}(\Xv,\tv_N)\no\\
&&\hspace{-8mm}=-\frac{\left[(d-\sigmav(2)c)\psi_{1}^{'}(z_1)+(\sigmav(2)a-b)\psi_{2}^{'}(z_1),\cdots,(d-\sigmav(2)c)\psi_{1}^{'}(z_N)+(\sigmav(2)a-b)\psi_{2}^{'}(z_N)\right]}{ad-bc}.\nonumber
\end{eqnarray}


Hence, using Theorem~\ref{thm:AIFgen}, for $p>1$, we have
\begin{eqnarray}
&&\hspace{-6mm}\AIF(\psiv,\Xv,p)\no\\
&=&\max\limits_{\sigmav}N^{1/p}\left|\left|\sigmav^T\left[\mathbf{G}^{'}_{\thetav}(\Xv,\tv_N)\right]^{-1}\mathbf{G}^{'}_{\Xv}(\Xv,\tv_N)\right|\right|_{\frac{p}{p-1}}\nonumber\\
&=&\frac{N^{1/p}}{|ad-bc|}\max\limits_{\sigmav(2)\in\{1,-1\}}\left[\sum\limits_{n=1}^{N}\left|(d-\sigmav(2)c)\psi_1^{'}(z_n)+(\sigmav(2)a-b)\psi_2^{'}(z_n)\right|^{p/(p-1)}\right]^{(p-1)/p}\nonumber\\
&=&\frac{N}{|ad-bc|}\max\limits_{\sigmav(2)\in\{1,-1\}}\left[\frac{1}{N}\sum\limits_{n=1}^{N}\left|(d-\sigmav(2)c)\psi_1^{'}(z_n)+(\sigmav(2)a-b)\psi_2^{'}(z_n)\right|^{p/(p-1)}\right]^{(p-1)/p}\nonumber\\
&\geq&\frac{N}{|ad-bc|}\max\limits_{\sigmav(2)\in\{1,-1\}}\frac{1}{N}\sum\limits_{n=1}^{N}|(d-\sigmav(2)c)\psi_1^{'}(z_n)+(\sigmav(2)a-b)\psi_2^{'}(z_n)|\nonumber\\
&{\geq}&\frac{1}{|ad-bc|}\max\limits_{\sigmav(2)\in\{1,-1\}}\left|\sum\limits_{n=1}^{N}\left((d-\sigmav(2)c)\psi_1^{'}(z_n)+(\sigmav(2)a-b)\psi_2^{'}(z_n)\right)\right|\nonumber\\
&{=}&1.\no
\end{eqnarray}

Summarizing the discussion, we have the following proposition.
\begin{prop}\label{prop:AIFgen}
	For joint location-scale estimation, 
	\begin{eqnarray}
	\AIF(\psiv,\Xv,p)=\frac{N^{1/p}}{|ad-bc|}\max\limits_{\sigmav(2)\in\{1,-1\}}\left[\sum\limits_{n=1}^{N}\left|(d-\sigmav(2)c)\psi_1^{'}(z_n)+(\sigmav(2)a-b)\psi_2^{'}(z_n)\right|^{p/(p-1)}\right]^{(p-1)/p}.\no
	\end{eqnarray}
	
	Furthermore, we have
	\begin{eqnarray}
	\AIF(\psiv,\Xv,p)\geq 1.\no
	\end{eqnarray}
\end{prop}

In the following, we provide several examples to analyze AIF of existing schemes.
\begin{exmpl}\label{ex:meanstandard}
	The coupled mean and sample standard deviation estimator (page 233 of~\citep{Hampel:Book:86}) is specified by
	$\psi_1(z)=z$ and $\psi_2(z)=z^2-1$. For this estimator, we have $\psi_1^{'}(z_n)=1$, $\psi_2^{'}(z_n)=2z_n$, and
	\begin{eqnarray}
	\sum\limits_{n=1}^Nz_n=0, \text{  and  }\sum\limits_{n=1}^N(z_n^2-1)=0.\no
	\end{eqnarray}
	As the result, 
	\begin{eqnarray}
	a&=&\sum\limits_{n=1}^N\psi_1^{'}(z_n)=N,\nonumber\\
	b&=&\sum\limits_{n=1}^Nz_n\psi_1^{'}(z_n)=\sum\limits_{n=1}^Nz_n=0,\nonumber\\
	c&=&\sum\limits_{n=1}^N\psi_2^{'}(z_n)=\sum\limits_{n=1}^N2z_n=0,\nonumber\\
	d&=&\sum\limits_{n=1}^Nz_n\psi_2^{'}(z_n)=\sum\limits_{n=1}^N2z_n^2=2N.\no
	\end{eqnarray}
	Hence, 
	\begin{eqnarray}
	\AIF(\psiv,\Xv,p)&=&\frac{N^{1/p}}{|ad-bc|}\max\limits_{\sigmav(2)\in\{1,-1\}}\left[\sum\limits_{n=1}^{N}\left|(d-\sigmav(2)c)\psi_1^{'}(z_n)+(\sigmav(2)a-b)\psi_2^{'}(z_n)\right|^{p/(p-1)}\right]^{(p-1)/p}\nonumber\\
	&=&\frac{N^{1/p}}{2N^2}\max\limits_{\sigmav(2)\in\{1,-1\}}\left[\sum\limits_{n=1}^{N}\left|2N+2N\sigmav(2)z_n\right|^{p/(p-1)}\right]^{(p-1)/p}\nonumber\\
	&=&\max\limits_{\sigmav(2)\in\{1,-1\}}\left[\frac{1}{N}\sum\limits_{n=1}^{N}\left|1+\sigmav(2)z_n\right|^{p/(p-1)}\right]^{(p-1)/p}.\no
	\end{eqnarray}
	In particular, when $p=2$, we have
	\begin{eqnarray}
	\AIF(\psiv,\Xv,2)&=&\max\limits_{\sigmav(2)\in\{1,-1\}}\left[\frac{1}{N}\sum\limits_{n=1}^{N}(1+\sigmav(2)z_n)^{2}\right]^{1/2}=\sqrt{2}.\no
	\end{eqnarray}
\end{exmpl}

\begin{exmpl}\label{ex:huber}
	Combination of the asymptotic minimax estimates of location and scale:
	\begin{eqnarray}
	\psi_1(z)=\max[-K,\min(K,z)],\text{  and  }
	\psi_2(z)=\min(\alpha^2,z^2)-\beta,\no
	\end{eqnarray}
	where $0<\beta<\alpha^2$. The parameters $K, \alpha$ and $\beta$ are properly chosen so that the corresponding location and scale estimators are minimax estimators respectively. If one further sets $\alpha=K$, this estimator corresponds to Huber's Proposal 2~\citep{Huber:AMS:64}.	For this joint estimator, we have
	\begin{eqnarray}
	\sum\limits_{n=1}^N\psi_1(z_n)&=&\sum\limits_{n\in \mathcal{A}_{-}:=\{z_n\leq -K\}}-K+\sum\limits_{n\in \mathcal{A}:=\{|z_n|<K\}}z_n+\sum\limits_{n\in \mathcal{A}_{+}:=\{z_n\geq K\}}K=0,\no\\
	\sum\limits_{n=1}^N\psi_2(z_n)&=&\sum\limits_{n\in \mathcal{B}:=\{z_n^2\leq c^2\}}(z_n^2-\beta)+\sum\limits_{n\in\bar{\mathcal{B}}}(\alpha^2-\beta)=0,\no
	\end{eqnarray}
	$
	\psi_1^{'}(z_n)=\mathbb{I}(n\in\mathcal{A})\text{ and }
	\psi_2^{'}(z_n)=2z_n\mathbb{I}(n\in\mathcal{B}).
	$
	
	As the result,
	\begin{eqnarray}
	a&=&\sum\limits_{n=1}^N\psi_1^{'}(z_n)=|\mathcal{A}|,\nonumber\\
	b&=&\sum\limits_{n=1}^Nz_n\psi_1^{'}(z_n)=\sum\limits_{n\in \mathcal{A}}z_n=K[|\mathcal{A}_{-}|-|\mathcal{A}_{+}|],\nonumber\\
	c&=&\sum\limits_{n=1}^N\psi_2^{'}(z_n)=2\sum\limits_{n\in\mathcal{B}}z_n,\nonumber\\
	d&=&\sum\limits_{n=1}^Nz_n\psi_2^{'}(z_n)=\sum\limits_{n\in\mathcal{B}}2z_n^2=2(N\beta-|\bar{\mathcal{B}}|c^2).\no
	\end{eqnarray}
	
	Plugging these into Proposition~\ref{prop:AIFgen}, we can obtain the formulator for AIF.
	
\end{exmpl}

In Section~\ref{sec:pop}, we will see that these existing schemes are not optimal in terms of minimizing AIF with or without additional IF constraints. We will derive optimal schemes that minimizes AIF with or without IF constraint in Section~\ref{sec:pop}.

\subsection{Population Case}\label{sec:pop}

In this section, we analyze the behavior of AIF as $N\rightarrow \infty$. Based on this analysis, we will characterize the optimal $\psiv$ that minimizes AIF. We will further identify a tradeoff between AIF and IF, and design $\psiv$ that achieves this optimal tradeoff.

Using Proposition~\ref{prop:AIFgen}, we have
\begin{eqnarray}
&&\hspace{-8mm}\AIF(\psiv,\Xv,2)\nonumber\\
&\hspace{-3mm}=&\hspace{-2mm}\frac{N}{|ad-bc|}\max\limits_{\sigmav(2)\in\{1,-1\}}\left[\frac{1}{N}\sum\limits_{n=1}^{N}\left|(d-\sigmav(2)c)\psi_1^{'}(z_n)+(\sigmav(2)a-b)\psi_2^{'}(z_n)\right|^{2}\right]^{1/2}\nonumber\\
&\hspace{-3mm}=&\hspace{-2mm}\frac{1}{|\tilde{a}\tilde{d}-\tilde{b}\tilde{c}|}\max\limits_{\sigmav(2)\in\{1,-1\}}\left[\frac{\sum\limits_{n=1}^{N}\left((\tilde{d}-\tilde{c})^2\psi_1^{'}(z_n)^2+2(\tilde{a}-\tilde{b})((\tilde{d}-\tilde{c}))\psi_1^{'}(z_n)\psi_2^{'}(z_n)+(\tilde{a}-\tilde{b})^2\psi_2^{'}(z_n)^2\right)}{N}\right]^{1/2}\label{eq:AIFsigma},
\end{eqnarray}
in which $\tilde{a}=\sigmav(2)a/N,\tilde{b}=b/N,\tilde{c}=\sigmav(2)c/N,\tilde{d}=d/N$. 

It has been shown in Theorem 2.4 of~\citep{Huber:Book:09} that, under certain mild regularity conditions, $T_N\overset{a.s.}{\rightarrow}\theta_1$ and $S_N\overset{a.s.}{\rightarrow}\theta_2$. In the following, we will need the following additional regularity conditions:
\begin{itemize}
	\item $\psi_1^{'}(z)$ and $\psi_2^{'}(z)$ are continuous functions. 
	\item There exist a function $K(z)$ such that $|\psi_1^{'}(z)|\leq K(z)$, $|z\psi_1^{'}(z)|\leq K(z)$, $|\psi_2^{'}(z)|\leq K(z)$, $|z\psi_2^{'}(z)|\leq K(z)$, and $\mathbb{E}[K(Z)]<\infty$.
\end{itemize} 
Recall that the CDF of $Z$ is $F_0$, hence $\mathbb{E}\{\cdot\}$ here (and the subsequent discussion) denotes expectation under $F_0$. The conditions here are slightly stronger than those conditions needed for the strong law of large numbers, as we will need to use the uniform strong law of large numbers (see Theorem 16 (a)~\citep{Ferguson:Book:96}). Under these regularity assumptions, using the uniform strong law of large numbers, Slutsky Theorem (see Chapter 6 of~\citep{Ferguson:Book:96}) and the fact that $T_N\overset{a.s.}{\rightarrow}\theta_1$ and $S_N\overset{a.s.}{\rightarrow}\theta_2$, we have 
\begin{eqnarray}
\tilde{a}&=&\sigmav(2)a/N\overset{a.s.}{\rightarrow} \sigmav(2)\mathbb{E}\{\psi_1^{'}(Z)\},\no\\
\tilde{d}&=&d/N\overset{a.s.}{\rightarrow} \mathbb{E}\{Z\psi_2^{'}(Z)\}.\no
\end{eqnarray}
Furthermore, 
\begin{eqnarray}
\tilde{b}=b/N\overset{a.s.}{\rightarrow} \mathbb{E}\{Z\psi_1^{'}(Z)\}=0,\no
\end{eqnarray}
as $\psi_1^{'}(z)$ is an even function (since $\psi_1(z)$ is an odd function) and $f_0$ is symmetric. Similarly, $\tilde{c}=\sigmav(2)c/N\overset{a.s.}{\rightarrow} \mathbb{E}\{\psi_2^{'}(Z)\}=0$ since $\psi_2(Z)$ is an even function.

As a result, 
\begin{eqnarray}
\AIF(\psiv,\Xv,2)\overset{a.s.}{\rightarrow} \sqrt{\frac{\mathbb{E}\{\psi_1^{'}(Z)^2\}}{(\mathbb{E}\{\psi_1^{'}(Z)\})^2}+\frac{\mathbb{E}\{\psi_2^{'}(Z)^2\}}{(\mathbb{E}\{Z\psi_2^{'}(Z)\})^2}}:=\AIF(\psiv,F_{0},2).\label{eq:AIFls}
\end{eqnarray}
In deriving this equation, we use the fact that $\psi_1^{'}(Z)\psi_2^{'}(Z)$ is an odd function, hence $$\frac{1}{N}\sum\limits_{n=1}^N\psi_1^{'}(z_n)\psi_2^{'}(z_n)\overset{a.s.}{\rightarrow} \mathbb{E}\{\psi_1^{'}(Z)\psi_2^{'}(Z)\}=  0.$$

\subsubsection{\textbf{Minimizing $\AIF(\psiv,F_{0},2)$ }}\label{sec:AIFnocon}

From the defender's perspective, one would like to design $\psiv$ such that $\AIF(\psiv,F_{0},2)$ is small. We now characterize $\psi_1$ and $\psi_2$ that minimize $\AIF(\psiv,F_{0},2)$. As~\eqref{eq:AIFls} can be decomposed into two independent terms, we can minimize $\psi_1(z)$ and $\psi_2(z)$ separately. Since $\psi_1(z)$ is an odd function and $\psi_2(z)$ is an even function, we only need to characterize the functions for $z\geq 0$.

To obtain the optimal $\psi_1(z)$, due to Jensen's inequality, we have that
\begin{eqnarray}
\frac{\mathbb{E}\{\psi_1^{'}(Z)^2\}}{(\mathbb{E}\{\psi_1^{'}(Z)\})^2}\geq 1\no
\end{eqnarray}
for which the equality holds when $\psi_1^{'}(z)$ is a constant. By setting $\psi_1(0)=0$, this function satisfies the requirements on $\psi_1(z)$ (i.e., $\psi_1(z)$ is Fisher consistent and is an odd function).

For $\psi_2(z)$, we need to solve
\begin{eqnarray}
\min && \frac{\mathbb{E}\{\psi_2^{'}(Z)^2\}}{(\mathbb{E}\{Z\psi_2^{'}(Z)\})^2}=\frac{\int_{0}^{\infty}\psi_2^{'}(z)^2f_0(z)\di z}{(\int_{0}^{\infty}z\psi_2^{'}(z)f_0(z)\di z)^2}, \label{eq:AIFscale} \\
\text{s.t.}&& \mathbb{E}\{\psi_2\}=2\int_{0}^{\infty}\psi_2(z)f_0(z)\di z=0,\label{eq:Fisherno2}\\
&& \psi_2^{'}(z)\geq 0.\no
\end{eqnarray}
in which we use the fact that $\psi_2^{'}(z)$ is an odd function and $f_0(z)$ is a symmetric function, while we use the requirement that $\psi_2(z)$ is an even function. Here,~\eqref{eq:Fisherno2} is the Fisher consistency requirement.
\begin{prop}\label{prop:aif}
	The optimal $\psi_2(z)$ is given by the following form: 
	\begin{itemize}
		\item $\psi_2^{'}(z)=\frac{z}{\int_{0}^{\infty}z^2f_0(z)\text{d}z}$.
		\item $\psi_2(0)=-\int_{0}^{\infty}\left(\int_{0}^{z}\psi_2^{'}(t)\di t\right)f_0(z)\di z$.
	\end{itemize}
\end{prop}
\begin{proof}
	Please see Appendix~\ref{app:aif}.
\end{proof}

\subsubsection{\textbf{Optimal $\AIF$ vs $\text{IF}$ Tradeoff}}\label{sec:AIFcon}

In Section~\ref{sec:AIFnocon}, we characterize the optimal $\psi_1$ and $\psi_2$ that minimize $\AIF(\psiv,F_{0},2)$. As we will see in the following, this choice will lead to an unbounded IF function. Hence, in this section, we aim to minimize $\AIF(\psiv,F_{0},2)$ while putting an upperbound on $\text{IF}$.

In particular, for the location-scale model, the influence function can be written as (see page 233 of~\citep{Hampel:Book:86} and using the Fisher consistent constraint)
\begin{eqnarray}
\text{IF}(z;T,F_0)=\left(\begin{array}{cc}
B_1^{-1}\psi_1(z)\\
B_2^{-1}\psi_2(z)
\end{array}\right),\no
\end{eqnarray}
in which
$
B_1=\mathbb{E}\{\psi_1^{'}(Z)\}\text{ and }
B_2=\mathbb{E}\{Z\psi_2^{'}(Z)\}.
$

Hence, the unstandardized gross-error sensitivity defined in~\eqref{eq:gross} for the location-scale model is given by
\begin{eqnarray}
\gamma_u^*=\sup\limits_{z}\sqrt{B_1^{-2}\psi_1^2(z)+B_2^{-2}\psi_2^2(z)}.
\end{eqnarray} 
It is easy to check that, if one uses $\psi_1(z)$ and $\psi_2(z)$ characterized in Section~\ref{sec:AIFnocon}, $\gamma_u^*$ is unbounded.

In the following, we aim to characterize $\psi_1(z)$ and $\psi_2(z)$ that minimize $\AIF(\psiv,F_{0},2)$ while making sure that $\gamma_u^*$ is bounded. In particular, we have the following optimization problem
\begin{eqnarray}
\min &&	\AIF(\psiv,F_{0},2),\label{eq:AIFcon}\\
\text{s.t.}&& \sup\limits_{z} ||\text{IF}(z;T,F_0)||_2^2\leq \xi^2,\no
\end{eqnarray}
for any given positive constraint $\xi$. As mentioned at the beginning of the section, we also require $\psi_1$ to be an odd function with nonnegative gradient and to be Fisher consistent, and require $\psi_2$ to be an even function with nonnegative gradient when $z\geq 0$ and to be Fisher consistent. Again, we only need to focus on the case when $z\geq 0$.

Plugging the expression of AIF and IF, we have
\begin{eqnarray}
\min && \sqrt{\frac{\mathbb{E}\{\psi_1^{'}(Z)^2\}}{(\mathbb{E}\{\psi_1^{'}(Z)\})^2}+\frac{\mathbb{E}\{\psi_2^{'}(Z)^2\}}{(\mathbb{E}\{Z\psi_2^{'}(Z)\})^2}}\label{eq:AIFIF}\\
\text{s.t.}&& \sup\limits_{z}\left[ \frac{\psi_{1}^2(z)}{(\mathbb{E}\{\psi_1^{'}(Z)\})^2}+ \frac{\psi_{2}^2(z)}{(\mathbb{E}\{Z\psi_2^{'}(Z)\})^2}\right]\leq \xi^2. \label{eq:IFcons}
\end{eqnarray}

As both $\psi_{1}^{'}(z)$ and $\psi_2^{'}(z)$ are nonnegative, the sup in~\eqref{eq:IFcons} are achieved when $z$ is either $\infty$ or $0$. Hence, using the constraint that $\psi_1(z)$ is odd, the IF constraint can be written as
\begin{eqnarray}
\max\left\{\frac{\psi_{2}^2(0)}{\mathbb{E}\{Z\psi_2^{'}(Z)\}^2},\no \frac{\psi_{2}^2(\infty)}{(\mathbb{E}\{Z\psi_2^{'}(Z)\})^2}+\frac{\psi_{1}^2(\infty)}{(\mathbb{E}\{\psi_1^{'}(Z)\})^2}\right\}\leq \xi^2.\no
\end{eqnarray}

By setting $\xi_1^2+\xi_2^2=\xi^2$, we can first solve the following two problems.
\begin{eqnarray}
\text{P1}:\hspace{6mm}\min && \frac{\mathbb{E}\{\psi_1^{'}(Z)^2\}}{(\mathbb{E}\{\psi_1^{'}(Z)\})^2} \no\\
\text{s.t.} 
&& \psi_1^{'}(z)\geq 0,\no\\
&& \frac{\psi_{1}^2(\infty)}{(\mathbb{E}\{\psi_1^{'}(Z)\})^2}\leq \xi_1^2,\no
\end{eqnarray}
and
\begin{eqnarray}
\text{P2:}\hspace{6mm}\min && \frac{\mathbb{E}\{\psi_2^{'}(Z)^2\}}{(\mathbb{E}\{Z\psi_2^{'}(Z)\})^2} \no\\
\text{s.t.}&& \mathbb{E}\{\psi_2\}=0,\no\\
&& \psi_2^{'}(z)\geq 0,\no\\
&& \frac{\psi_{2}^2(\infty)}{(\mathbb{E}\{Z\psi_2^{'}(Z)\})^2}\leq \xi_2^2,\no\\
&&\frac{\psi_{2}^2(0)}{(\mathbb{E}\{Z\psi_2^{'}(Z)\})^2}\leq \xi^2.\no
\end{eqnarray}

After solving these problems for a given $\xi_1^2$ and $\xi_2^2$, we can then adjust the values $\xi_1^2$ and $\xi_2^2$ to obtain the overall solution to the optimization problem~\eqref{eq:AIFIF}. In P1, we do not write the Fisher consistent constraint, as an odd function $\psi_1(z)$ will automatically satisfy this constraint for symmetric $f_0(z)$.

\begin{thm}\label{prop:aifcon}
	The optimal odd function $\psi_{1}(z)$ is specified by
	\begin{eqnarray}
	\psi_1^{'}(z)=\left[\nu^*-\frac{ \vartheta_1^*}{ f_0(z)}\right]^+,\label{eq:psi1}
	\end{eqnarray}
	in which $\nu^*$ and $\vartheta_1^*$ are chosen to satisfy
	\begin{eqnarray}
	&&\int_{0}^{\infty}\psi_{1}^{'}(z)f_0(z)\text{d}z=1,\nonumber\\
	&&\vartheta^*_1\left( \int_{0}^{\infty}\psi_{1}^{'}(z)\text{d}z
	-2 \xi_1 \right )=0.\no
	\end{eqnarray}
	
	The optimal even function $\psi_2(z)$ have the following form:
	\begin{eqnarray}
	\psi_2^{'}(z)=\left[
	\nu^* z-\frac{ \vartheta_2^*+(\vartheta_1^*-\vartheta_2^*)F_0(z)}{ f_0(z)}\right]^+\label{eq:psi2}
	\end{eqnarray}
	and 
	$\psi_2(0)=-\int_{0}^{\infty}\psi_2^{'}(z)(1-F_0(z))\di z$, in which
	the parameters $\nu^*$, $\vartheta^*_1\geq 0$ and $\vartheta^*_2\geq 0$ satisfy
	\begin{eqnarray}
	&&\int_{0}^{\infty}z\psi_2^{'}(z)f_0(z)\text{d}z=1,\label{eq:xpsi2}\\
	&&\vartheta^*_1\left( \int_{0}^{\infty}\psi_2^{'}(z)F_0(z)\text{d}z
	- 2\xi_2 \right )=0,\label{eq:nu1}\\
	&& \vartheta^*_2\left(\int_{0}^{\infty}\psi_2^{'}(z)\left[1-F_0(z)\right]\text{d}z-2\xi\right)=0.\label{eq:vartheta2}
	\end{eqnarray}

\end{thm}
\begin{proof}
	Please see Appendix~\ref{app:aifcon}.	
\end{proof}

\section{Numerical Examples}\label{sec:simulation}
In this section, we provide numerical examples to illustrate the results obtained in this paper.
\subsection{Joint Location-Scale Estimation}
In this example, we consider the adversarially robust joint location-scale estimation for Laplace random variables. In this case, $f_0$ has the following form
\begin{eqnarray}
f_0(z)=\frac{1}{2}\exp(-|z|).\no
\end{eqnarray}

From Proposition~\ref{prop:aif}, we know that the optimal $\psiv$ that minimizes $\text{AIF}(\psiv,F_0,2)$ has the following form:
\begin{enumerate}
	\item $\psi_1(0)=0$, $\psi_1^{'}(z)=c$ with $c$ being a constant.
	\item $\psi_2(0)=-1$ and $\psi_2^{'}(z)=2z$.
\end{enumerate}

In the following, using Theorem~\ref{prop:aifcon}, we characterize the optimal $\psiv$ that minimizes $\text{AIF}(\psiv,F_0,2)$ subject to the constraint that the $\ell_2$ norm of $\text{IF}(x;T,F_0)$ is upper-bounded by $\xi$. We will only state the form of the functions for $z\geq 0$ as $\psi_1(z)$ is an odd function and $\psi_2(z)$ is an even function.

First, from~\eqref{eq:psi1}, we know that the optimal form $\psi_1^{'}(z)$ has the following form: 
\begin{eqnarray}
\psi_1^{'}(z)=\nu^*-2\vartheta_1^*\exp(z),\hspace{3mm} 0\leq z\leq a_1,\no
\end{eqnarray}
in which $a_1$ is the solution of $$\frac{\nu^*}{2\vartheta_1^*}=\exp(z),$$ and $(\nu^*,\vartheta_1^*)$ are chosen to satisfy
\begin{eqnarray}
&&\int_{0}^{a_1}\left(\nu^*\exp(-z)-2\vartheta_1^*\right)\text{d}z=2,\no\\
&& \int_{0}^{a_1}\left(\nu^*-2\vartheta_1^*\exp(z)\right)\text{d}z
= 2\xi_1.\no
\end{eqnarray}

It is easy to check that these conditions can be simplified to 
\begin{eqnarray}
&&\nu^*-2(a_1+1)\vartheta_1^*=2,\no\\
&& (a_1-1)\nu^*+2\vartheta_1^*=2\xi_1.\no
\end{eqnarray}
Using these, we can express the values of $\nu^*$ and $\vartheta_1^*$ in terms of $a_1$
\begin{eqnarray}
\nu^*&=&\frac{2+2(a_1+1)\xi_1}{a_1^2},\label{eq:nu}\\
\vartheta_1^*&=&\frac{\xi_1-(a_1-1)}{a_1^2},\label{eq:var}
\end{eqnarray}
and the value of $a_1$ is determined by
\begin{eqnarray}
e^{a_1}=\frac{1+(a_1+1)\xi_1}{\xi_1+1-a_1}.\label{eq:a}
\end{eqnarray}
For a given $\xi_1$, the value of $a_1$ can be obtained by solving~\eqref{eq:a} numerically. We can then plug the value of $a$ into~\eqref{eq:nu} and~\eqref{eq:var} to obtain $\nu^*$ and $\vartheta_1^*$. 

Secondly, for a given $\xi_2$, we determine the form of $\psi_2$ in Theorem \ref{prop:aifcon}. For $f_0$ considered in Laplace random variables, we know that $1-F_0(z)\leq F_0(z)$ when $z\geq 0$. Using this fact in~\eqref{eq:nu1} and~\eqref{eq:vartheta2} along with the fact that $\xi_2\leq\xi$, we know that $\vartheta_2^*=0$ and $\vartheta_1^*\neq 0$. Hence from Theorem \ref{prop:aifcon}, we know that $\psi_2^{'}(z)$ has the following form
\begin{eqnarray}
\psi_2^{'}(z)=
\nu^* z-\vartheta_1^*(2\exp(z)-1), \hspace{3mm} 0\leq a_2\leq z\leq b,\no
\end{eqnarray}
in which both $a_2$ and $b$ satisfy \begin{eqnarray}
\frac{\nu^*}{\vartheta_1^*} =\frac{2\exp(z)-1}{z}.\label{eq:ab}
\end{eqnarray}
Furthermore $\nu^*$ and $\vartheta_1^*$ are determined by the following two equations (simplified from~\eqref{eq:nu1} and~\eqref{eq:vartheta2})
\begin{eqnarray}
\int_{a_2}^{b}\left(\nu^* z-\vartheta_1^*(2\exp(z)-1)\right)z\exp(-z)\di z=2,\no\\
\int_{a_2}^{b}\left(\nu^* z-\vartheta_1^*(2\exp(z)-1)\right)\left(2-\exp(-z)\right)\di z=4\xi_2.\no
\end{eqnarray}
After simple integral and using the fact that $a_2$, $b$ satisfy~\eqref{eq:ab}, we can simplify the above two equations to
\begin{eqnarray}
\frac{\nu^*}{\vartheta_1^*}=\frac{2(e^b-e^{a_2})-2(\xi_2+1)(b-a_2)+\xi_2(a_2^2-b^2+e^{-b_2}-e^{-a})}{0.5(b^2-a_2^2)+(2\xi_2+0.5)(e^{-b}-e^{-a_2})},\label{eq:nuv}\\
\vartheta_1^*\left[\frac{\nu^*}{\vartheta_1^*}2(e^{-b}-e^{-a_2})+b^{2}-a_2^{2}-e^{-b}+e^{-a_2}+2(b-a_2)\right]=-2.\label{eq:var1}
\end{eqnarray}

To find the value of $\nu^*$ and $\vartheta_1^*$, we can first obtain $\nu^*/\vartheta_1^*$ (and hence the values of $a_2$ and $b$) numerically from~\eqref{eq:ab} and~\eqref{eq:nuv}. After knowing $\nu^*/\vartheta_1^*$, we can then use~\eqref{eq:var1} to obtain $\vartheta_1^*$.

\begin{figure}[h]
	\begin{center}
		\includegraphics[width=12cm]{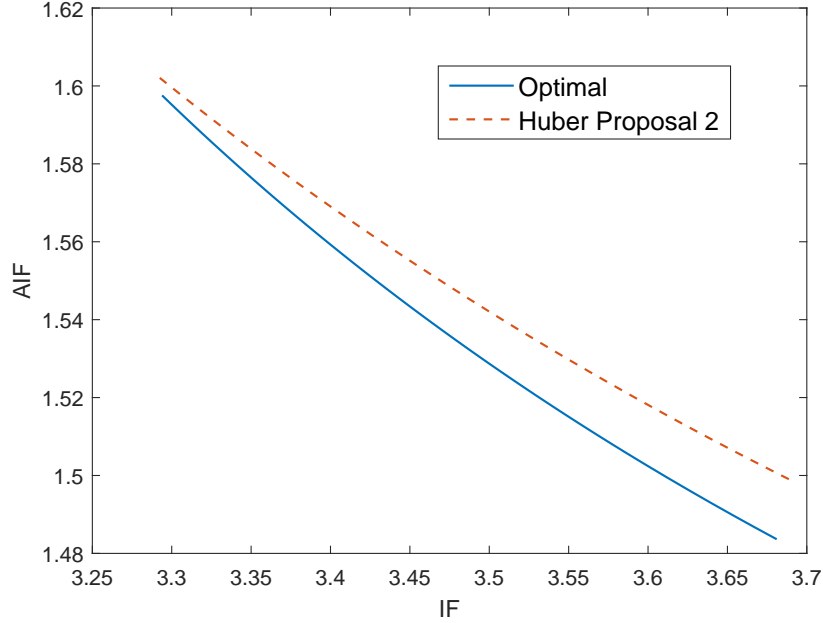}
		\caption{The optimal tradeoff between AIF and IF for joint location-scale estimation in Laplace random variables}
		\label{fig:scalelocationtradeoff}
	\end{center}
\end{figure}

Fig~\ref{fig:scalelocationtradeoff} illustrates the tradeoff between AIF and IF of the optimal $M$-estimator characterized using the approach outlined above. In the figure, we also plot the curve for Huber proposal 2 discussed in Example~\ref{ex:huber}. From the figure, we can see that as IF increases (less robust to outliers), AIF decreases (more robust to adversary modifications) and vice verse. Furthermore, there is a gap between the tradeoff achieved by the Huber proposal 2 and the optimal tradeoff achieved by the estimator characterized above. The tradeoff achieved by Huber proposal 2 in turn is better than that of the coupled mean and standard deviation estimator discussed in Example~\ref{ex:meanstandard}, for which achieves $\text{AIF}=\sqrt{2}$ but $\text{IF}=\infty$.

\subsection{Robust Linear Regression}

In this section, we compare AIF of various robust regression methods discussed in Section~\ref{sec:linearregression}. In the following, we will adopt the commonly used coefficients $w_n=\sqrt{1-h_{nn}}$, $v_n=1/w_n$ and use Huber function for $\eta$, that is \begin{eqnarray}\eta(x)=\max[-K,\min(K,x)].\label{eq:Huber}\end{eqnarray} To proceed further, we need to compute 
\begin{eqnarray}
\frac{\partial w_n}{\partial \xv_{j,k}}&=&\frac{-1}{2\sqrt{1-h_{nn}}}\frac{\partial h_{nn}}{\partial \xv_{j,k}},\label{eq:wpartial}\\ \frac{\partial v_n}{\partial \xv_{j,k}}&=&\frac{1}{2(1-h_{nn})^{3/2}}\frac{\partial h_{nn}}{\partial \xv_{j,k}},\label{eq:vpartial}
\end{eqnarray} 
both of which depend on $\partial h_{nn}/\partial \xv_{j,k}$. Recall that 
$h_{nn}=\xv_n^T(\Xv\Xv^T)^{-1}\xv_n$, which is a complicated function of $\xv_{j,k}$. To address this, we use $\Xv_{(-j)}$ to denote the data matrix $\Xv$ but with the $j$th column removed. To simplify notation, we let $\Av=\Xv\Xv^T$ and $\Av_{(-j)}=\Xv_{(-j)}\Xv_{(-j)}^T$. We have $\Av=\Xv\Xv^T=\Xv_{(-j)}\Xv_{(-j)}^T+\xv_j\xv_j^T=\Av_{(-j)}+\xv_j\xv_j^T$. Using Sherman-Morrison formula, we have
\begin{eqnarray}
\Av^{-1}=(\Av_{(-j)}+\xv_j\xv_j^T)^{-1}=\Av_{(-j)}^{-1}-\frac{\Av_{(-j)}^{-1}\xv_j\xv_j^T\Av_{(-j)}^{-1}}{1+\xv_j^T\Av_{(-j)}^{-1}\xv_j}.\no
\end{eqnarray}
Hence,
\begin{eqnarray}
h_{nn}=\xv_n^T\Av^{-1}\xv_n=\xv_n^T\Av_{(-j)}^{-1}\xv_n-\frac{\left(\xv_n^T\Av_{(-j)}^{-1}\xv_j\right)^2}{1+\xv_j^T\Av_{(-j)}^{-1}\xv_j}.\no
\end{eqnarray}

As the result, if $j=n$, we have
\begin{eqnarray}
\frac{\partial h_{nn}}{\partial \xv_{n,k}}=\frac{2\Av_{(-n)}^{-1}(k,:)\xv_n}{\left(1+\xv_n^T\Av_{(-n)}^{-1}\xv_n\right)^2}.\no
\end{eqnarray}
If $j\neq n$, we have
\begin{eqnarray}
\frac{\partial h_{nn}}{\partial \xv_{j,k}}=-\frac{2\Av_{(-n)}^{-1}(k,:)\xv_n\left(1+\xv_j^T\Av_{(-j)}^{-1}\xv_j\right)-2\left(\xv_n^T\Av_{(-j)}^{-1}\xv_j\right)^2\Av_{(-n)}^{-1}(k,:)\xv_j}{\left(1+\xv_j^T\Av_{(-j)}^{-1}\xv_j\right)^2}.\no
\end{eqnarray}

Plugging~\eqref{eq:wpartial} and~\eqref{eq:vpartial} into the corresponding equations of OLS, Huber's proposal, Mallow's proposal and Schewppe's proposal discussed in Section~\ref{sec:linearregression}, we obtain the corresponding AIF.  

\begin{figure}[h]
	\begin{center}
		\includegraphics[width=12cm]{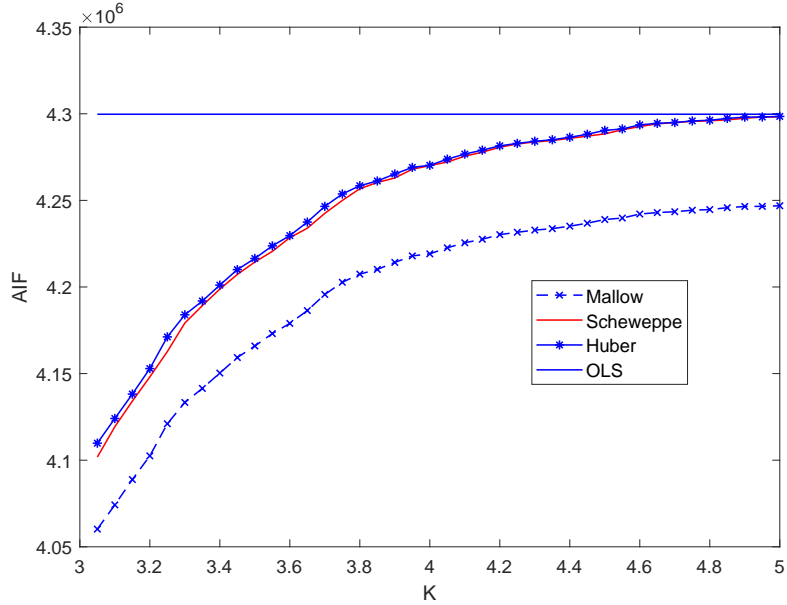}
		\caption{Adversarial robustness of Scheweppe and Mallow types of robust linear regression}
		\label{fig:schvsmallow}
	\end{center}
\end{figure}

Figure~\ref{fig:schvsmallow} illustrates the comparison of AIF for these methods for different value of $K$, the parameter in the Huber function~\eqref{eq:Huber}. In generating this figure, we set $q=5$, $N=500$. We first generate $\thetav$ using Gaussian zero mean and variance 1. After $\thetav$ is generated, it is fixed throughout the simulation. We let $K$ to be from 3.05 to 5. For each position, we run 100 times and obtain the average. We generate each entry of $\Xv$ using i.i.d with zero mean and variance 1. We then obtain $y_n$ by adding zero mean variance 2 noise to $\thetav^T\xv_n$. From the figure, we can see that for the same value of $K$, Mallow's proposal has the smallest value of AIF (i.e., it is the most robust again adversary modifications), the AIF value of Scheweppe's proposal is similar to Huber's proposal. Furthermore, all three approaches are more adversary modification resistant than OLS, which does not depend on $K$. Furthermore, as $K$ increases, the performance of all three methods approach that of OLS. This is expected, as $K$ increases, all three approaches are similar to OLS.

\section{Conclusion}\label{sec:con}

In this paper, we have investigated the adversarial robustness of multivariate $M$-Estimators. We have characterized the adversary's optimal modification strategy and its corresponding AIF. Under certain regularity conditions, we have characterized the optimal $M$-estimator for the case of joint location-scale estimation problem. We have further identified a tradeoff between robustness against adversarial modification and robustness against outliers, and have derived the optimal $M$-estimator that achieves the best tradeoff.

\appendix

\section{Proof of $f^*=g^*$}\label{app:fg}
First, we show $f^*\geq g^*$.

To simplify the presentation, we let $\bv_i$, $i=1,\cdots,q$ be the $i$-th row of $\tv_N^{'}(\Xv)$. For any $\sigmav\in\Sigmav$, we have
\begin{eqnarray}
||\tv_N^{'}(\Xv)\text{Vec}(\Delta\Xv)||_1=\sum\limits_{i=1}^q|\bv_i \text{Vec}(\Delta\Xv)|\nonumber\\
\overset{(a)}{=}\sum\limits_{i=1}^q|\sigmav(i)\bv_i \text{Vec}(\Delta\Xv)|\nonumber\\
\geq \sum\limits_{i=1}^q \sigmav(i)\bv_i \text{Vec}(\Delta\Xv),
\end{eqnarray}	
in which $(a)$ is true as each entry of $\sigmav$ is either 1 or $-1$.
As this holds for any $\sigmav\in \Sigmav$, hence $f^*\geq g^*$.

Next we show $f^*\leq g^*$.

$||\tv_N^{'}(\Xv)\text{Vec}(\Delta\Xv)||_1=\sum\limits_{i=1}^q|\bv_i \text{Vec}(\Delta\Xv)|=\sum\limits_{i=1}^q \text{sign}(\bv_i \text{Vec}(\Delta\Xv))\bv_i \text{Vec}(\Delta\Xv)$. Hence

\begin{eqnarray}
f^*&=&\max\limits_{\text{Vec}(||\Delta \Xv)||_p^p\leq Nm\delta^p}||\tv_N^{'}(\Xv)\text{Vec}(\Delta\Xv)||_1\no\\
&=&
\max\limits_{\sigmav\in\Sigmav}\hspace{5mm}\max_{\text{sign}(\bv_i \text{Vec}(\Delta\Xv))=\sigmav(i),||\Delta \Xv)||_p^p\leq Nm\delta^p}\sum\limits_{i=1}^q \sigmav(i)\bv_i \text{Vec}(\Delta\Xv)\no\\
&\leq& \max\limits_{\sigmav\in\Sigmav}\max_{||\Delta \Xv)||_p^p\leq Nm\delta^p}\sum\limits_{i=1}^q \sigmav(i)\bv_i \text{Vec}(\Delta\Xv)=g^*.\no
\end{eqnarray}

\section{Proof of Proposition~\ref{prop:aif}}\label{app:aif}
From~\eqref{eq:AIFscale}, we have the following variational optimization problem
\begin{eqnarray}
\min && \frac{\int_{0}^{\infty}\psi_2^{'}(z)^2f_0(z) \di z}{(\int_{0}^{\infty}z\psi_2^{'}(z)f_0(z)\di z)^2},\label{eq:cost}	\\
\text{s.t.}&& \int_{0}^{\infty}\psi_2(z)f_0(z)\di z=\psi_2(0)+\int_{0}^{\infty}\left(\int_{0}^{z}\psi_2^{'}(t)\di t\right)f_0(z)\di z=0,\label{eq:constraint}\\
&& \psi_{2}^{'}\geq 0.\no
\end{eqnarray}

As $\psi_2(0)$ does not appear in the objective function, we can solve~\eqref{eq:cost} without the constraint~\eqref{eq:constraint} first. After that, we can simply set $$\psi_2(0)=-\int_{0}^{\infty}\left(\int_{0}^{z}\psi_2^{'}(t)\di t\right)f_0(z)\di z$$ so that the constraint~\eqref{eq:constraint} will be satisfied. 

In the following, to simplify the notation, we will use $g(z)$ to denote $\psi_2^{'}(z)$. It is clear that the optimization problem is scale invariant in the sense that if $g^*(z)$ is a solution to this problem, then for any positive constant $c$,  $cg^*(z)$ is also a solution to this problem. As a result, without loss of generality, we can assume $\int_{0}^{\infty}zg(z)f_0(z)\text{d}z=1$. Using this, we can further simplify the optimization problem to 
\begin{eqnarray}
\min &&\frac{1}{2}\int_{0}^{\infty}g^2(z)f_0(z)\text{d}z,\nonumber\\
\text{s.t.}&& \int_{0}^{\infty}zg(z)f_0(z)\text{d}z=1,\nonumber\\
&& g(z)\geq 0.\no
\end{eqnarray}
For this convex calculus of variations problem, we form Lagrange function
\begin{eqnarray}
\mathcal{L}=\int_{0}^{\infty}\frac{1}{2}g^2(z)f_0(z)\text{d}z+\nu\left(-\int_{0}^{\infty}zg(z)f_0(z)\text{d}z+1\right)-\int_{0}^{\infty}\lambda(z)g(z)\di z.\nonumber
\end{eqnarray}

For any given $z$, the optimal value of $g(z)$ can be found from
\begin{eqnarray}
g^*(z)f_0(z)-\nu^* zf_0(z)-\lambda^*(z)=0,\label{eq:gfscale}
\end{eqnarray}
in which the parameters $\nu^*$ and $\lambda^*(z)\geq 0$ satisfy~\citep{Kot:Book:14, Gregory:Book:92}
\begin{eqnarray}
&&\int_{0}^{\infty}g^*(z)f_0(z)\text{d}z=1,\nonumber\\
&& \lambda^*(z)g(z)= 0.\label{eq:posiscale}
\end{eqnarray}

From~\eqref{eq:gfscale}, for $z\geq 0$ in the range of $f_0(z)$, we have
\begin{eqnarray}
g^*(z)=\frac{\lambda^*(z)+\nu^* zf_0(z)}{f_0(z)}.\nonumber
\end{eqnarray}

As $z\geq 0$ and $f_0\geq 0$, we know from~\eqref{eq:posiscale} that $\lambda^*(z)=0$, and hence
\begin{eqnarray}
g^*(z)=\nu^* z.\nonumber
\end{eqnarray}
and the optimal value of $\nu^*$ is 
\begin{eqnarray}
\nu^*=\frac{1}{\int_{0}^{\infty}z^2f_0(z)\text{d}z}.\nonumber
\end{eqnarray}
As the result, for $z$ in the range of $f_0(z)$, the optimal $g^*(z)$ is 
\begin{eqnarray}
g^*(z)=\frac{z}{\int_{0}^{\infty}z^2f_0(z)\text{d}z},\nonumber
\end{eqnarray}
and 
$\psi_2(0)=-\int_{0}^{\infty}\left(\int_{0}^{z}g^{*}(t)\di t\right)f_0(z)\di z$.

\section{Proof of Theorem~\ref{prop:aifcon}}\label{app:aifcon}

We first focus on P1, and rewrite P1 into the following form
\begin{eqnarray}
\text{P1a}:\hspace{6mm}\min && \frac{\mathbb{E}\{\psi_1^{'}(Z)^2\}}{(\mathbb{E}\{\psi_1^{'}(Z)\})^2}=\frac{\int_{0}^{\infty}\psi_1^{'}(z)^2f_0(z)\di z}{2(\int_{0}^{\infty}\psi_1^{'}(z)f_0(z)\di z)^2}, \no\\
\text{s.t.} 
&& \psi_1^{'}(z)\geq 0,\no\\
&& \frac{\psi_{1}^2(\infty)}{\left(\mathbb{E}\{\psi_1^{'}(Z)\}\right)^2}=\frac{\left(\int_{0}^{\infty}\psi_1^{'}(z)\di x\right)^2}{4(\int_{0}^{\infty}\psi_1^{'}(z)f_0(z)\di z)^2}\leq \xi_1^2.\no
\end{eqnarray}
Here, we use the constraint that $\psi_1(z)$ is an odd function.

To ease the notation, in the following, we use $h(z)$ to denote $\psi_1^{'}(z)$. It is clear that the optimization problem is scale invariant, hence we can assume $\int_{0}^{\infty}h(z)f_0(z)\di z=1$. Hence, P1a can be simplified to
\begin{eqnarray}
\min && \frac{1}{2}\int_{0}^{\infty}h^2(z)f_0(z)\di z, \no\\
\text{s.t.} && h(z)\geq 0,\no\\
&& \int_{0}^{\infty}h(z)\di z\leq 2\xi_1,\no\\
&& \int_{0}^{\infty}h(z)f_0(z)\di z=1.\no
\end{eqnarray}

To solve this convex functional minimization problem, we first form the Lagrangian function
\begin{eqnarray}
\mathcal{L}&=&\frac{1}{2}\int_{0}^{\infty}h^2(z)f_0(z)\text{d}z+\nu\left(-\int_{0}^{\infty}h(z)f_0(z)\text{d}z+1\right)-\int_{0}^{\infty}\lambda(z)h(z)\text{d}z\nonumber\\
&&\hspace{-10mm}+\vartheta_1\left(\int_{0}^{\infty}h(z)\text{d}z- 2\xi_1\right).\nonumber
\end{eqnarray}

For any given $z$, the optimal value of $h(z)$ can be found from
\begin{eqnarray}
h^*(z)f_0(z)-\nu^* f_0(z)+\vartheta_1^*-\lambda^*(z)=0,\label{eq:gf}
\end{eqnarray}
in which the parameters $\nu^*$, $\vartheta_1^*\geq 0$ and $\lambda^*(z)\geq 0$ satisfy~\citep{Kot:Book:14, Gregory:Book:92}
\begin{eqnarray}
&&\int_{0}^{\infty}h^*(z)f_0(z)\text{d}z=1,\nonumber\\
&&\vartheta^*_1\left( \int_{0}^{\infty}h^*(z)\text{d}z
-2 \xi_1 \right )=0,\nonumber\\
&&\lambda^*(z)h(z)= 0.\label{eq:posi}
\end{eqnarray}

From~\eqref{eq:gf}, for $z\geq 0$ in the range of $f_0(z)$, we have
\begin{eqnarray}
h^*(z)=\frac{\lambda^*(z)+\nu^* f_0(z)-\vartheta^*_1}{f_0(z)}.\nonumber
\end{eqnarray}

Combining this with condition~\eqref{eq:posi}, we know that if $\nu^* f_0(z)-\vartheta^*_1>0$, then $\lambda^*(z)=0$. On the other hand, if $\nu^* f_0(z)-\vartheta^*_1<0$, then $h^*(z)=0$. As a result, we have
\begin{eqnarray}
h^*(z)=\left\{\begin{array}{cc}
\nu^*-\frac{ \vartheta_1^*}{ f_0(z)}, & \nu^* f_0(z)> \vartheta_1^* ;\\
0,& \text{otherwise}.
\end{array}\right.\nonumber
\end{eqnarray}
This characterizes the optimal odd function $\psi_1(z)$.

We now focus on $\psi_{2}$. We rewrite P2 into
\begin{eqnarray}
\text{P2a:}\hspace{6mm}\min && \frac{\mathbb{E}\{\psi_2^{'}(Z)^2\}}{(\mathbb{E}\{Z\psi_2^{'}(Z)\})^2}=\frac{\int_{0}^{\infty}\psi_2^{'}(z)^2f_0(z)\di z}{2(\int_{0}^{\infty}z\psi_2^{'}(z)f_0(z)\di z)^2},\no \\
\text{s.t.}&& \mathbb{E}\{\psi_2\}=2\psi_2(0)+2\int_{0}^{\infty}\left[\int_{0}^{z}\psi_2^{'}(t)\di t\right]f_0(z)\di z=0,\label{eq:Fisherno}\\
&& \psi_2^{'}(z)\geq 0,\no\\
&& \frac{\psi_{2}^2(\infty)}{\left(\mathbb{E}\{Z\psi_2^{'}(Z)\}\right)^2}=\frac{\left(\psi_{2}(0)+\int_{0}^{\infty}\psi_2^{'}(z)\di z\right)^2}{4\left(\int_{0}^{\infty}z\psi_2^{'}(z)f_0(z)\di z\right)^2}\leq \xi_2^2,\no\\
&&\frac{\psi_{2}^2(0)}{\left(\mathbb{E}\{Z\psi_2^{'}(Z)\}\right)^2}=\frac{\psi_{2}^2(0)}{4\left(\int_{0}^{\infty}z\psi_2^{'}(z)f_0(z)\di z\right)^2}\leq \xi^2.\no
\end{eqnarray}

We note that in P2a,
\begin{eqnarray}
\int_{0}^{\infty}\left[\int_{0}^{z}\psi_2^{'}(t)\di t\right]f_0(z)\di z=\int_{0}^{\infty}\psi_2^{'}(z)(1-F_0(z))\di z.\no
\end{eqnarray}
To satisfy the Fisher consistent constraint~\eqref{eq:Fisherno} in P2a, we need to set $$\psi_2(0)=-\int_{0}^{\infty}\psi_2^{'}(z)(1-F_0(z))\di z.$$
Using these, P2a can be simplified to
\begin{eqnarray}
\text{P2b:}\hspace{6mm}\min && \frac{\mathbb{E}\{\psi_2^{'}(Z)^2\}}{(\mathbb{E}\{Z\psi_2^{'}(Z)\})^2}=\frac{\int_{0}^{\infty}\psi_2^{'}(z)^2f_0(z)\di z}{2(\int_{0}^{\infty}z\psi_2^{'}(z)f_0(z)\di z)^2}, \no\\
\text{s.t.}&& \psi_2^{'}(z)\geq 0,\no\\ &&\frac{\left(-\int_{0}^{\infty}\psi_2^{'}(z)(1-F_0(z))\di z+\int_{0}^{\infty}\psi_2^{'}(z)\di z\right)^2}{4\left(\int_{0}^{\infty}z\psi_2^{'}(z)f_0(z)\di z\right)^2}\leq \xi_2^2,\no\\
&&\frac{(-\int_{0}^{\infty}\psi_2^{'}(z)(1-F_0(z))\di z)^2}{4\left(\int_{0}^{\infty}z\psi_2^{'}(z)f_0(z)\di z\right)^2}\leq \xi^2.\no
\end{eqnarray}

Similar to other cases, it is clear that P2b is scale invariant, and hence we can without loss of generality assume that $\int_{0}^{\infty}z\psi_2^{'}(z)f_0(z)dz=1$. Using this fact and denoting $g(z)=\psi_2^{'}(z)$, P2b can be simplified to
\begin{eqnarray}
\text{P2b:}\hspace{6mm}\min && \frac{1}{2}\int_{0}^{\infty}g(z)^2f_0(z)\di z \no \\
\text{s.t.}&& \int_{0}^{\infty}g(z)F_0(z)\di z\leq 2\xi_2,\no\\
&&\int_{0}^{\infty}g(z)(1-F_0(z))\di z\leq 2\xi,\no\\
&& \int_{0}^{\infty}zg(z)f_0(z)\di z=1,\no\\
&& g(z)\geq 0.\no
\end{eqnarray}
To solve this convex functional minimization problem, we first form the Lagrangian function
\begin{eqnarray}
\mathcal{L}&=&\hspace{-3mm}\frac{1}{2}\int_{0}^{\infty}g^2(z)f_0(z)\text{d}z+\nu\left(-\int_{0}^{\infty}zg(z)f_0(z)\text{d}z+1\right)-\int_{0}^{\infty}\lambda(z)g(z)\text{d}z\nonumber\\
&&\hspace{-3mm}+\vartheta_1\left(\int_{0}^{\infty}g(z)F_0(z)\text{d}z
- 2\xi_2 \right)+\vartheta_2\left(\int_{0}^{\infty}g(z)\left[1-F_0(z)\right]\text{d}z- 2\xi\right).\nonumber
\end{eqnarray}

For any given $z$, the optimal value of $g(z)$ can be found from
\begin{eqnarray}
g^*(z)f_0(z)-\nu^* z f_0(z)+\vartheta^*_2+(\vartheta^*_1-\vartheta^*_2)F_0(z)-\lambda^*(z)=0,\label{eq:gfs}
\end{eqnarray}
in which the parameters $\nu^*$, $\vartheta^*_1\geq 0$, $\vartheta^*_2\geq 0$, $\lambda^*(z)\geq 0$ satisfy
\begin{eqnarray}
&&\int_{0}^{\infty}xg^*(z)f_0(z)\text{d}z=1,\nonumber\\
&&\vartheta^*_1\left( \int_{0}^{\infty}g^*(z)F_0(z)\text{d}z
- 2\xi_2 \right )=0,\nonumber\\
&& \vartheta^*_2\left(\int_{0}^{\infty}g^*(z)\left[1-F_0(z)\right]\text{d}z- 2\xi\right)=0,\nonumber\\
&& \lambda^*(z)g^*(z)= 0.\label{eq:posis}
\end{eqnarray}

From~\eqref{eq:gfs}, for those $z\geq 0$ with $f_0(z)>0$, we have
\begin{eqnarray}
g^*(z)=\frac{\lambda^*(z)+\nu^* z f_0(z)-\vartheta^*_2-(\vartheta^*_1-\vartheta^*_2)F_0(z)}{f_0(z)}.\nonumber
\end{eqnarray}

Combining this with the condition~\eqref{eq:posis}, we know that if $\nu^* z f_0(z)-\vartheta^*_2-(\vartheta^*_1-\vartheta^*_2)F_0(z)>0$, then $\lambda^*(z)=0$. On the other hand, if $\nu^* f_0(z)-\vartheta^*_2-(\vartheta^*_1-\vartheta^*_2)F_0(z)<0$, then $g^*(z)=0$. As the result, we have

\begin{eqnarray}
g^*(z)=\left\{\begin{array}{cc}
\nu^* z-\frac{ \vartheta_2^*+(\vartheta_1^*-\vartheta_2^*)F_0(z)}{ f_0(z)}, & \nu^* z f_0(z)> \vartheta_1^* F_0(z)+\vartheta_2^*(1-F_0(z));\nonumber\\
0,& \text{otherwise}.
\end{array}\right.
\end{eqnarray}
Coupled with $\psi_2(0)=-\int_{0}^{\infty}g^{*}(z)(1-F_0(z))\di z$, this characterizes the optimal even $\psi_2(z)$.

\section*{Acknowledgements}
The work of E. Bayraktar was supported in part by the National Science Foundation under grant DMS-1613170 and by the Susan M. Smith Professorship. The work of L. Lai was supported by the National Science Foundation under grants CCF-17-17943 and ECCS-17-11468.


\end{document}